\documentclass[10pt]{article}
\usepackage[a4paper, left=0.8in, right=0.8in, top=1in, bottom=1in]{geometry}
\usepackage{amsmath, amssymb, amsthm, mathtools}
\usepackage{graphicx}
\usepackage{booktabs}
\usepackage{microtype}
\usepackage[numbers,sort&compress]{natbib}
\usepackage{hyperref}
\hypersetup{colorlinks=true,linkcolor=blue,citecolor=blue,urlcolor=blue}
\usepackage{placeins}
\usepackage[nameinlink,capitalize]{cleveref}

\crefname{theorem}{Theorem}{Theorems}
\crefname{lemma}{Lemma}{Lemmas}
\crefname{proposition}{Proposition}{Propositions}
\crefname{corollary}{Corollary}{Corollaries}
\crefname{definition}{Definition}{Definitions}
\crefname{remark}{Remark}{Remarks}

\newcommand{\R}{\mathbb{R}}
\newcommand{\N}{\mathbb{N}}

\newcommand{\ip}[2]{\left\langle #1, #2 \right\rangle}
\newcommand{\norm}[1]{\left\lVert #1 \right\rVert}

\newcommand{\X}{\mathcal{X}}

\newtheorem{theorem}{Theorem}[section]
\newtheorem{lemma}[theorem]{Lemma}
\newtheorem{proposition}[theorem]{Proposition}
\newtheorem{corollary}[theorem]{Corollary}
\newtheorem{definition}[theorem]{Definition}
\newtheorem{remark}{Remark}[section]

\title{%
\includegraphics[width=0.3\linewidth]{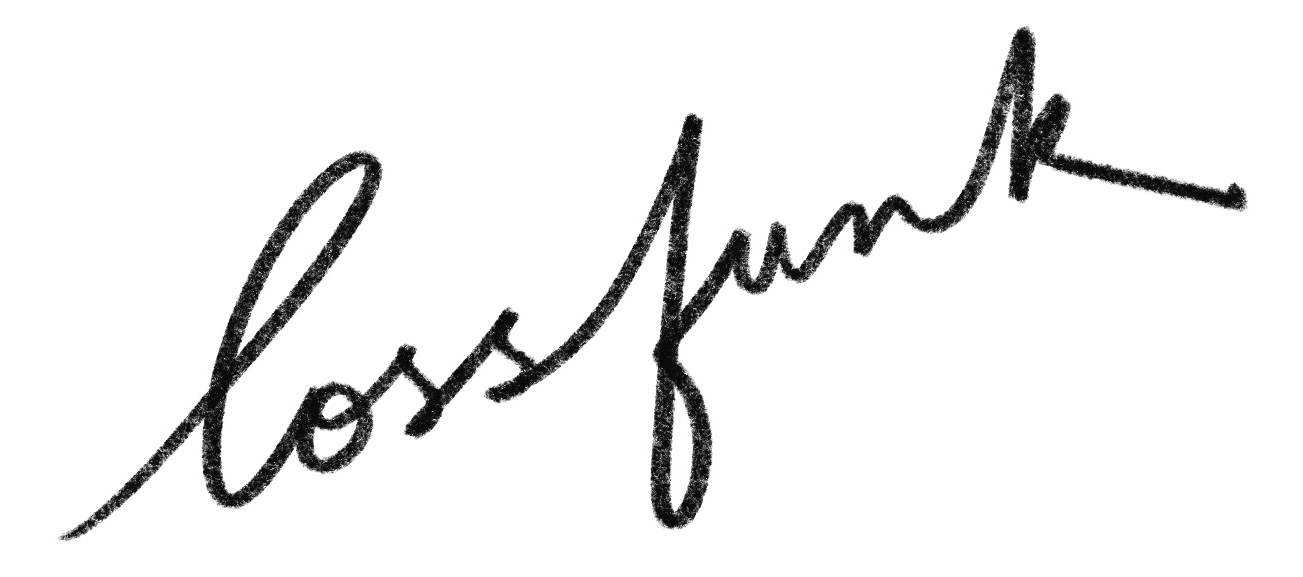} 
\\[1cm]
Small-Gain Nash: Certified Contraction to Nash Equilibria\\
in Differentiable Games}
\author{Vedansh Sharma \\
  \texttt{vedansh.research@gmail.com}}
\date{}

\begin{document}
\maketitle

\begin{abstract}
Classical convergence guarantees for gradient-based learning in games
require the pseudo-gradient to be (strongly) monotone in Euclidean
geometry~\citep{rosen1965,facchinei2003}, a condition that often fails even in
simple games with strong cross-player couplings.
We introduce \emph{Small-Gain Nash} (SGN), a block small-gain condition
in a custom block-weighted geometry.
SGN converts local curvature and cross-player Lipschitz coupling bounds
into a tractable \emph{certificate of contraction}. It constructs a
weighted block metric in which the pseudo-gradient becomes strongly
monotone on any region where these bounds hold, even when it is
non-monotone in the Euclidean sense.
The continuous flow is exponentially
contracting in this designed geometry, and projected Euler and RK4 discretizations converge under
explicit step-size bounds derived from the SGN margin and a local
Lipschitz constant.
Our analysis reveals a certified ``timescale band'', a non-asymptotic, metric-based certificate that plays a TTUR-like role: rather than forcing asymptotic timescale separation via vanishing, unequal step sizes, SGN identifies a finite band of relative metric weights for which a single-step-size dynamics is provably contractive.
We validate the framework on quadratic games where Euclidean monotonicity analysis fails to predict convergence, but SGN successfully certifies it, and extend the construction to mirror/Fisher geometries for entropy-regularized policy gradient in Markov games.
The result is an offline certification pipeline that estimates curvature,
coupling, and Lipschitz parameters on compact regions, optimizes block
weights to enlarge the SGN margin, and returns a structural, computable convergence
certificate consisting of a metric, contraction rate, and safe
step-sizes for non-monotone games.
\end{abstract}

\setlength{\parskip}{1em}
\setlength{\parindent}{0pt}

\section{Introduction}

Gradient-based learning in multi-player and adversarial settings is ubiquitous, from stylized
adversarial training problems and latent-space games to multi-agent reinforcement learning and economic
models. The resulting dynamics are typically driven by the pseudo-gradient of the players' objectives,
and one would like both the continuous-time flow and simple discrete-time schemes to converge robustly
to a Nash equilibrium. Classical convergence results assume that the pseudo-gradient operator is
(strongly) monotone in Euclidean geometry~\citep{rosen1965,facchinei2003}, an assumption violated even
by simple two-player quadratic games with convex per-player objectives and strong couplings, where
simultaneous gradient descent oscillates or diverges.

This work takes a structural, geometry-design viewpoint rather than proposing a new game algorithm.
We ask, for which differentiable games can we design a block-diagonal metric in which standard gradient-based
dynamics are contracting, and hence converge at a quantified rate to a Nash equilibrium on a certified
region?

Our starting point is a block-diagonal metric $P = \mathrm{diag}(P_i)$, where each block $P_i \succ 0$
weights player $i$'s variables, together with a positive weight vector $w \in \mathbb{R}^N_{++}$ that
rescales players relative to each other. The resulting geometry is encoded by
$M(w) := \mathrm{diag}(w_i P_i)$, and we say that the game is contracting with margin $\alpha > 0$ if
the pseudo-gradient operator $F$ is $\alpha$-strongly monotone in the $M(w)$-inner product, or
equivalently if the continuous flow $\dot{x} = -F(x)$ is $\alpha$-contracting in $\|\cdot\|_{M(w)}$.

We introduce a structural condition, \emph{Small-Gain Nash} (SGN), that uses player-wise curvature and
cross-player coupling bounds to certify such contraction in block metrics. Once a metric and SGN margin
$\alpha > 0$ have been certified on a region where these bounds hold, we show that the pseudo-gradient
flow and simple explicit schemes (projected Euler and classical RK4) are contracting in this geometry
under explicit step-size bounds expressed in terms of $\alpha$ and a Lipschitz constant $\beta$ of
$G := -F$. The pair $(\alpha,\beta)$ plays the role of a game-theoretic CFL number that translates
directly into safe step sizes on the certified region. Throughout, all structural bounds are stated on
closed convex regions where curvature, coupling, and Lipschitz conditions hold (often compact
forward-invariant sets). These are the certified regions on which our ``global'' statements should be
understood.

\paragraph{Contributions.}
At a high level, we do not propose a new game algorithm; instead, we design a geometry in which existing
gradient-based dynamics become contracting. Our main contributions are:

\begin{itemize}
\item We introduce the Small-Gain Nash (SGN) condition, a block-diagonal small-gain inequality in a
weighted block metric that, via the matrix inequality $C(w,\alpha) \succ 0$, converts curvature and
coupling bounds into strong monotonicity of the pseudo-gradient, extending Rosen's diagonal concavity and
complementing gain-matrix contraction criteria for Nash dynamics~\citep{gokhale2023contractivity} to general block metrics.

\item We show that SGN yields existence and uniqueness of a variational-inequality/Nash solution on
certified regions and exponential contraction of pseudo-gradient flow in the SGN metric, so that the continuous-time game dynamics are globally well-posed and convergent within this designed geometry.

\item We show that SGN induces a finite ``timescale band'' in the space of player weights $w$: as long as the geometric ratio $r=w_2/w_1$ lies in this band, the scalar-step pseudo-gradient dynamics is contractive in the SGN metric. In contrast to TTUR, which stabilizes by sending the step-size ratio to infinity and often imposing a large timescale separation, SGN can certify that no TTUR-style slowdown is needed for a class of games that look unstable in Euclidean geometry.

\item We develop an offline metric-and-timescale certification pipeline that, on a compact region,
estimates curvature, couplings, and Lipschitz constants from Hessian/Jacobian probes, optimizes over
weights $w$ to obtain a certificate $(\alpha,\beta,M(w))$, and validates the framework on non-monotone
quadratic games and a tabular Markov game, via mirror/Fisher and Markov-game extensions.
\end{itemize}

\section{Preliminaries}

\subsection{Games, pseudo-gradients, and Nash equilibria}

Let $N\in\N$ be the number of players, with convex, closed strategy sets $\X_i\subseteq\R^{d_i}$ and continuously differentiable cost functions $f_i:\X\to\R$ on $\X:=\prod_{i=1}^N \X_i$.
We interpret $f_i$ as a cost (loss) to be minimized by player~$i$.
The \emph{pseudo-gradient} is the block operator
\[
  F(x) := \begin{bmatrix}
    \nabla_{x_1} f_1(x)\\
    \vdots\\
    \nabla_{x_N} f_N(x)
  \end{bmatrix}
  \in\R^d,\qquad d=\sum_{i=1}^N d_i.
\]
A point $x^*\in\X$ is a (first-order) Nash equilibrium if for each $i$,
\[
  \ip{\nabla_{x_i} f_i(x^*)}{x_i - x_i^*} \;\ge\; 0
  \qquad\text{for all }x_i\in\X_i.
\]
Given an SPD matrix $M\succ 0$ with inner product $\ip{\cdot}{\cdot}_M$, the variational inequality $\mathrm{VI}_M(F,\X)$ asks for $x^*\in\X$ such that
\[
  \ip{F(x^*)}{x-x^*}_{M} \;\ge\; 0
  \qquad\text{for all }x\in\X.
\]
When $M=I$ we write $\mathrm{VI}(F,\X)$; on product sets and under the usual convexity assumptions on $\X_i$ and $f_i(\cdot,x_{-i})$, solutions of $\mathrm{VI}(F,\X)$ coincide with Nash equilibria~\citep[Section~1.4]{facchinei2003}.
Since $M\succ 0$ is invertible, $\mathrm{VI}_M(F,\X)$ is equivalent, via a linear change of variables, to a Euclidean VI for a transformed operator.
Thus $M$ serves primarily as an analysis tool; for unconstrained games, interior equilibria, or cases where $M$ is compatible with the constraints (e.g., diagonal $M$ with box constraints), solutions of $\mathrm{VI}_M(F,\X)$ coincide with standard Nash equilibria.

All structural bounds (curvature, couplings, Lipschitz constants) will be stated on a closed convex region $\mathcal{R}\subseteq\X$, often a compact forward-invariant set, on which they hold; unless noted otherwise, ``global'' statements refer to such a certified region (or to all of $\R^d$ when the bounds hold everywhere).

\subsection{Block metrics and induced norms}

For each player $i$, let $P_i\in\R^{d_i\times d_i}$ be symmetric positive definite (SPD), and let $P:=\mathrm{diag}(P_1,\dots,P_N)$.
For $v=(v_1,\dots,v_N)\in\R^d$ we define the $P$-norm and inner product by
\[
  \norm{v}_P^2 := \sum_{i=1}^N v_i^\top P_i v_i,\qquad
  \ip{u}{v}_P := \sum_{i=1}^N u_i^\top P_i v_i.
\]
Allowing arbitrary $P_i$ captures block preconditioning and non-Euclidean geometries (e.g., Fisher metrics on simplices); the Euclidean case corresponds to $P_i=I_{d_i}$.
The matrices $P_i$ encode per-player geometry, while the player weights $w$ introduced below will rescale these blocks to form the SGN metric $M(w)$.

\paragraph{Mixed operator norms.}
For a linear map $A\in\R^{d_i\times d_j}$, the mixed operator norm induced by $(P_i,P_j)$ is
\[
  \norm{A}_{P_j\to P_i} \;:=\;
  \sup_{v\ne 0} \frac{\norm{A v}_{P_i}}{\norm{v}_{P_j}}.
\]
Equivalently,
\[
  \norm{A}_{P_j\to P_i} \;=\; \big\|P_i^{1/2} A P_j^{-1/2}\big\|_2,
\]
so this is just the operator $2$-norm in weighted coordinates.
These norms will be used to bound cross-player couplings via Jacobian and Hessian blocks.

\subsection{Block curvature and cross couplings}

We quantify player-wise curvature and cross-player couplings via parameters $\mu_i>0$ and $L_{ij}\ge 0$.

\begin{definition}[Block bounds in $P$-geometry]
\label{def:block-bounds}
We say that the pseudo-gradient $F$ obeys block bounds $(\mu,L)$ in $P$-geometry on a set $\X\subseteq\R^d$ if for all $x,y\in\X$ and all $i$,
\begin{equation}
  \ip{x_i-y_i}{\nabla_{x_i} f_i(x) - \nabla_{x_i} f_i(y)}_{P_i}
  \;\ge\;
  \mu_i \norm{x_i-y_i}_{P_i}^2
  \;-\;
  \sum_{j\ne i} L_{ij}\,\norm{x_i-y_i}_{P_i}\,\norm{x_j-y_j}_{P_j}.
  \label{eq:block-bounds}
\end{equation}
\end{definition}
These inequalities summarize player-wise curvature and cross-player couplings in the $P$-geometry: if each $f_i$ is twice continuously differentiable, with own-player Hessians satisfying $\nabla^2_{x_i x_i} f_i(x)\succeq \mu_i P_i$ and mixed partials bounded by $\|\nabla^2_{x_i x_j} f_i(x)\|_{P_j\to P_i}\le L_{ij}$ for all $x\in\X$, then~\eqref{eq:block-bounds} holds and $(\mu_i,L_{ij})$ can be read as curvature and coupling parameters.
In applications, we will use these bounds on a closed, convex region $\mathcal{R}\subseteq\X$ on which they have been certified, often a compact forward-invariant set.

\paragraph{Notation summary.}
We summarize a few recurring symbols that will be used throughout:
\begin{itemize}
  \item $P_i\succ 0$: per-player SPD metric; $P=\mathrm{diag}(P_i)$.
  \item $w\in\R_{++}^N$: player weights; $M(w)=\mathrm{diag}(w_i P_i)$ is the SGN block metric.
  \item $\X=\prod_i \X_i$: joint strategy space; $\mathcal{R}\subseteq\X$: certified region where structural bounds hold.
  \item $d=\sum_{i=1}^N d_i$: total dimension of the joint strategy profile.
  \item $F$: pseudo-gradient; $G=-F$.
  \item $(\mu_i,L_{ij})$: block curvature and coupling parameters in \cref{def:block-bounds}.
  \item $C(w,\alpha)$: SGN small-gain matrix from~\eqref{eq:Cwalpha}.
  \item $H(w)$: normalized gain matrix in \cref{subsec:normalized-gain}.
  \item $S_M(x)$: $M$-symmetric Jacobian part from~\eqref{eq:dsc}.
\end{itemize}

\section{Small-Gain Nash: Definition and Geometry}

\subsection{Weighted small-gain condition}

We now introduce the Small-Gain Nash condition.
Let $L\in\R^{N\times N}$ collect the cross-player coupling parameters from \cref{def:block-bounds}, with $L_{ii}=0$ and $L_{ij}\ge 0$ for $i\ne j$.
For a weight vector $w\in\R_{++}^N$ and margin parameter $\alpha\ge 0$, define the symmetric matrix $C(w,\alpha)\in\R^{N\times N}$ by
\begin{equation}
  C_{ii}(w,\alpha) := 2w_i(\mu_i-\alpha),\qquad
  C_{ij}(w,\alpha) := -\big(w_i L_{ij} + w_j L_{ji}\big)\quad (i\ne j).
  \label{eq:Cwalpha}
\end{equation}

\begin{definition}[Small-Gain Nash (SGN)]
\label{def:sgn}
Given block bounds $(\mu,L)$ as in \cref{def:block-bounds} and a weight vector $w\in\R_{++}^N$, we say that $\mathrm{SGN}(\mu,L;P,w)$ holds with margin $\alpha>0$ if $C(w,\alpha)\succ 0$.
The associated (weight-dependent) strong monotonicity margin is
\begin{equation}
  \alpha_*(w) := \sup\big\{\alpha\ge 0:\, C(w,\alpha)\succ 0\big\}.
  \label{eq:alpha-w}
\end{equation}
\end{definition}

Condition~\eqref{eq:alpha-w} is a weighted small-gain matrix inequality tying the
per-player curvature bounds $\mu_i$ and couplings $L_{ij}$ through
the SGN matrix $C(w,\alpha)$.
In addition to this exact spectral condition $C(w,\alpha) \succ 0$,
one can certify SGN via simpler Gershgorin-style diagonal-dominance
tests.
In particular, by Gershgorin’s theorem applied to $C(w,0)$ and
shifting by $\alpha I$, we obtain the conservative lower bound
\[
  \alpha_*(w)
  \;\ge\;
  \min_i\Big(
    \mu_i - \frac{1}{2 w_i}
      \sum_{j\neq i} (w_i L_{ij} + w_j L_{ji})
  \Big),
\]
so any $\alpha$ not exceeding the right-hand side is admissible.
Intuitively, $C(w,\alpha)\succ 0$ enforces a weighted diagonal dominance condition:
for each player $i$, the reweighted own-player curvature $w_i\mu_i$ dominates the aggregate cross couplings $w_iL_{ij}+w_jL_{ji}$.

\subsection{Strong monotonicity in the SGN geometry}

Let $W:=\mathrm{diag}(w_1 I_{d_1},\dots,w_N I_{d_N})$ and define the block metric
\[
  M(w) := \mathrm{diag}(w_i P_i) = WP.
\]
We write $\norm{\cdot}_{M(w)}$ and $\ip{\cdot}{\cdot}_{M(w)}$ for the norm and inner product induced by $M(w)$.

\begin{lemma}[Strong monotonicity from SGN]
\label{lem:strong-mono}
Suppose the block bounds~\eqref{eq:block-bounds} hold on a set $\mathcal{R}\subseteq\X$ and the weighted SGN condition of \cref{def:sgn} is satisfied for some $w\in\R_{++}^N$ with margin $\alpha>0$, i.e., $C(w,\alpha)\succ 0$.
Then for all $x,y\in\mathcal{R}$,
\begin{equation}
  \ip{x-y}{F(x)-F(y)}_{M(w)}
  \;\ge\; \alpha\,\norm{x-y}_{M(w)}^2.
  \label{eq:strong-mono}
\end{equation}
In particular, $F$ is $\alpha$-strongly monotone in the $M(w)$-geometry.
The proof (\cref{app:proofs}) rewrites the expression
\[
  \ip{x-y}{F(x)-F(y)}_{M(w)} - \alpha\norm{x-y}_{M(w)}^2
\]
as the quadratic form
$\tfrac{1}{2} a^\top C(w,\alpha)a$ in the block distances $a_i:=\norm{x_i-y_i}_{P_i}$;
positive definiteness of $C(w,\alpha)$ then yields~\eqref{eq:strong-mono}.
\end{lemma}

\begin{corollary}[Unique VI and Nash solution]
\label{cor:unique-nash}
Under the conditions of \cref{lem:strong-mono}, if we take $\mathcal{R}=\X$ with $\X$ nonempty, closed, and convex, then the variational inequality $\mathrm{VI}(F,\X)$ has a unique solution $x^*$.
If each $\X_i$ is convex and $f_i(\cdot, x_{-i})$ is convex for all $i$, then $x^*$ is the unique Nash equilibrium.
This follows by transporting $\mathrm{VI}(F,\X)$ to an equivalent Euclidean VI under the change of variables induced by $M(w)$ and applying standard monotone VI theory; see \cref{app:proofs}.
\end{corollary}

Since $F$ is $\alpha$-strongly monotone in the $M(w)$-geometry, solutions of the continuous flow $\dot x=-F(x)$ satisfy $\norm{x(t)-x^*}_{M(w)}\le e^{-\alpha t}\,\norm{x(0)-x^*}_{M(w)}$; see, e.g., standard monotone operator and convex analysis treatments such as~\citep{boyd1994,bauschke2011}.

In addition to the exact SGN matrix condition $C(w,\alpha)\succ 0$, one can certify SGN via simple Gershgorin-style diagonal-dominance conditions on a normalized gain matrix; these are collected in \cref{subsec:normalized-gain}.

\subsection{Best SGN margin}
\label{subsec:best-sgn-margin}

For a fixed weight vector $w\in\R_{++}^N$ and region $\mathcal{R}\subseteq\X$ on which the block bounds and SGN parameters are valid, the SGN margin $\alpha_*(w)$ from~\eqref{eq:alpha-w} is equivalently the largest number for which
\[
  \ip{x-y}{F(x)-F(y)}_{M(w)} \;\ge\; \alpha\,\norm{x-y}_{M(w)}^2
  \qquad\forall x,y\in\mathcal{R}.
\]
We define the \emph{best diagonal SGN margin} on $\mathcal{R}$ by
\begin{equation}
  \alpha^*(F;\mathcal{R})
  := \sup_{w\in\R_{++}^N} \alpha_*(w)
  = \sup\bigl\{\alpha\ge 0:\,\exists\,w\succ 0\text{ with }C(w,\alpha)\succ 0\bigr\},
  \label{eq:alpha-star}
\end{equation}
which measures how much strong monotonicity can be recovered on $\mathcal{R}$ by optimally choosing a block-diagonal metric of the form $M(w)$.
This is the best strong monotonicity margin obtainable within this diagonal SGN family; in general, non-block metrics could yield larger margins, but we restrict to this structured class for tractability.

On a region $\mathcal{R}$ where the block bounds~\eqref{eq:block-bounds} hold, the SGN matrix $C(w,\alpha)$ encodes a block small-gain inequality: $C(w,\alpha)\succ 0$ is equivalent to strong monotonicity of $F$ with margin $\alpha$ in the metric $M(w)$ by \cref{lem:strong-mono}.
The normalized gain matrix $H(w)$ and Gershgorin condition of \cref{subsec:normalized-gain} provide a conservative but easy-to-check sufficient condition for $C(w,\alpha)\succ 0$ that depends only on $(\mu,L)$ and the weight ratios $w_i/w_j$.

\begin{remark}[Sufficiency and conservatism of SGN]
The SGN condition is a sufficient, not necessary, certificate of stability.
Even when the pseudo-gradient is strongly monotone (or the flow is contracting) in some SPD metric, there need not exist a weight vector $w$ such that the block-diagonal SGN matrix $C(w,\alpha)$ is positive definite. In such cases, more general Jacobian-based directional stability / DSC conditions can certify larger regions or larger margins than SGN. We briefly develop this refinement and its relation to SGN in \cref{app:proofs}.
The appeal of SGN is that it trades metric expressivity for tractability: it reduces the certificate to coarse block summaries $(\mu_i,L_{ij})$ and a structured block-diagonal metric $M(w)$, avoiding the need to work with the full Jacobian spectrum while incurring some conservatism.
In particular, any margin $\alpha$ certified by SGN guarantees that the DSC condition~\eqref{eq:dsc} holds in the same metric with at least margin $\alpha$ (see \cref{app:proofs}); DSC can then refine this lower bound locally, while SGN provides a tractable, block-structured way to obtain a computable certificate in the first place.
\end{remark}

\subsection{Generalization to non-Euclidean geometries}

The SGN condition in \cref{def:sgn} depends only on block curvature and coupling parameters $(\mu_i,L_{ij})$ measured in a chosen block norm.
These quantities are geometry-agnostic: they can be defined equally well for any choice of strictly convex potentials $\psi_i$ and associated local norms or Bregman divergences $D_{\psi_i}$ on each player's strategy space.
In particular, the small-gain construction extends to mirror geometries on simplices, with Bregman distances and Fisher metrics replacing Euclidean norms.
We formalize this mirror-SGN variant and its Lyapunov-based contraction guarantees in \cref{subsec:mirror-sgn} and Appendix~\ref{app:mirror-sgn}.

\section{Discrete-Time Dynamics: Euler and Runge-Kutta}

We now show that explicit discretizations of the pseudo-gradient flow $\dot x=-F(x)$ inherit contraction in a fixed SPD metric $M$ under suitable step-size conditions.
Throughout this section we fix a nonempty closed convex set $\mathcal{R}\subseteq\R^d$ (for example, $\mathcal{R}=\X$ or a compact forward-invariant region) on which $G=-F$ is Lipschitz in $\norm{\cdot}_{M}$ with constant $\beta>0$, i.e.,
\begin{equation}
  \norm{G(x)-G(y)}_{M} \;\le\; \beta\,\norm{x-y}_{M}
  \qquad\forall x,y\in\mathcal{R}.
  \label{eq:G-Lip}
\end{equation}
These discrete-time contraction results are standard for contractive flows under log-norm analysis and nonexpansive projections; we recall them here in a general metric $M$ and refer to, e.g.,~\citep{soderlind2006,hairer1993,butcher2003,gottlieb2001} for general background.
When applying them in the SGN geometry we simply take $M=M(w)$.

\begin{remark}[Strong monotonicity versus Lipschitzness]
If $F$ is $\alpha$-strongly monotone and $G=-F$ is $\beta$-Lipschitz in $\norm{\cdot}_{M}$ on a set, then necessarily $\beta\ge\alpha$.
In our step-size bounds, the pair $(\alpha,\beta)$ plays a CFL-type role: $\alpha$ sets the contraction strength of the flow in $\norm{\cdot}_M$ and $\beta$ the curvature/Lipschitz scale, which is why projected Euler and RK4 naturally use steps on the order of $\alpha/\beta^2$ and $1/\beta$, respectively (cf.\ \cref{thm:proj-euler,thm:rk4}).
\end{remark}

\subsection{Projected forward-Euler}

Let $\Pi_{\X}^{M}$ denote the metric projection onto a nonempty closed convex set $\X$ under the norm $\norm{\cdot}_{M}$.

\begin{lemma}[Nonexpansiveness of metric projection]
\label{lem:P-proj}
\[
  \norm{\Pi_{\X}^{M}(x) - \Pi_{\X}^{M}(y)}_{M}
  \;\le\; \norm{x-y}_{M}.
\]
\end{lemma}

Consider the projected forward-Euler scheme
\[
  \mathcal{T}_\eta(x)
  := \Pi_{\X}^{M}\big(x + \eta G(x)\big),
  \qquad \eta>0.
\]

\begin{theorem}[Projected Euler contraction]
\label{thm:proj-euler}
Assume $F$ is $\alpha$-strongly monotone in $\norm{\cdot}_M$ on a nonempty convex set $\mathcal{R}$ (for example, as certified by \cref{lem:strong-mono} with $M=M(w)$) and that $G$ satisfies~\eqref{eq:G-Lip} on $\mathcal{R}$.
For step-sizes $\eta\in\big(0, \tfrac{2\alpha}{\beta^2}\big)$, the map $\mathcal{T}_\eta$ is a contraction in $\norm{\cdot}_{M}$ on $\X\cap\mathcal{R}$ with factor
\[
  \sqrt{1 - 2\alpha\eta + \beta^2\eta^2} \;<\; 1.
\]
\end{theorem}

\begin{remark}[Relation to existing Euler contraction results]
Results of the form in \cref{thm:proj-euler} are classical in contraction theory and monotone operator theory: for a vector field that is $\alpha$-contracting and $\beta$-Lipschitz in a fixed metric, explicit Euler is contracting for step sizes proportional to $2\alpha/\beta^2$; see, e.g., contraction-based analyses of optimization flows in \citep{cisneros2021contraction} and the Euler discretization results in \citep{centorrino2024euler,gokhale2024proximal}.
\cref{thm:proj-euler} is a specialization of these general bounds to projected pseudo-gradient dynamics in the metric $M$, stated in the $(\alpha,\beta)$ notation used throughout this paper.
\end{remark}

\subsection{Explicit Runge-Kutta methods}

We now discretize the pseudo-gradient flow $\dot x=-F(x)$ with explicit Runge-Kutta methods.
Our focus is on the classical explicit Runge-Kutta 4 (RK4) method; using the logarithmic norm $\mu_M$ associated with a metric $M$ (defined via the $M$-symmetric Jacobian in \cref{app:proofs}), a standard log-norm argument for explicit RK schemes (see, e.g.,~\citep[Section~II.7]{hairer1993} and~\citep[Section~3]{soderlind2006}) yields the following contraction bound (\cref{thm:rk4}).
In practice one often applies \cref{thm:proj-euler,thm:rk4} on a forward-invariant region $\mathcal{S}$ on which a directional/DSC margin $\alpha$ and a Lipschitz constant $\beta$ have been certified in the metric $M$.
If $\mathcal{S}$ is also invariant under the discrete-time maps (e.g., by projection onto a convex subset $\X\cap\mathcal{S}$), then the contraction arguments below apply on $\mathcal{S}$ and all iterates remain in the certified region.

\subsubsection*{RK4 contraction}
\begin{theorem}[RK4 contraction under a log-norm step bound]
\label{thm:rk4}
Suppose $G=-F$ is $\beta$-Lipschitz in $\norm{\cdot}_M$ on a convex set $\mathcal{R}$ and the DSC condition~\eqref{eq:dsc} holds there with margin $\alpha>0$, i.e., $\mu_M(J_G(x))\le -\alpha$ for all $x\in\mathcal{R}$.
Then, by the general log-norm contractivity results for explicit Runge-Kutta methods (see, e.g.,~\citep[Section~II.7]{hairer1993} and~\citep[Section~3]{soderlind2006} applied to the classical RK4 stability function $R(z)=1+z+z^2/2+z^3/6+z^4/24$), there exist method-dependent constants $C_4>0$ and $c_4\in(0,1]$ such that for any step size
\begin{equation}
  0 < h \le \frac{C_4}{\beta},
  \label{eq:hbound}
\end{equation}
RK4 is a contraction in the $\norm{\cdot}_M$-norm:
\[
  \norm{\Phi_h(x) - \Phi_h(y)}_{M}
  \;\le\; \exp(-c_4\alpha h)\,\norm{x-y}_{M}
  \qquad \forall x,y\in\mathcal{R}.
\]
\end{theorem}

\begin{remark}[Numerical values for RK4 constants]
For the classical RK4 tableau, one can numerically verify conservative choices such as $C_4\approx 2.5$ and $c_4\approx \tfrac{1}{2}$, consistent with the real-axis stability interval $[-\gamma_4,0]$ where $\gamma_4\approx 2.785$; see, e.g.,~\citep{hairer1993,butcher2003}.
In the theory we only require the existence of such method-dependent constants; our experiments use numerically validated values.
\end{remark}

\begin{corollary}[Local certified balls and forward invariance]
\label{cor:ball-forward}
Let $x^*$ be the unique VI/Nash solution from \cref{cor:unique-nash}, and let $\mathcal{R}\subseteq\X$ be a region on which the assumptions of \cref{thm:proj-euler} or \cref{thm:rk4} hold.
Suppose the corresponding one-step map $\mathcal{T}$ (projected Euler or RK4) is a contraction on $\mathcal{R}$ in $\norm{\cdot}_{M(w)}$ with factor $q\in(0,1)$.
For any radius $r>0$ such that the closed ball
$B_r := \{x:\,\norm{x-x^*}_{M(w)}\le r\}$ satisfies $B_r\subseteq\mathcal{R}$, the ball $B_r$ is forward-invariant under $\mathcal{T}$ and, for any $x^0\in B_r$,
\[
  \norm{x^k - x^*}_{M(w)} \;\le\; q^k\,\norm{x^0-x^*}_{M(w)}
  \qquad\text{for all }k\ge 0.
\]
\end{corollary}

\subsection{Projected integrators in the SGN metric}
\label{subsec:integrator-cookbook}

\noindent
For contractive flows, standard results in numerical analysis state that explicit integrators such as projected Euler and RK4 remain stable provided the step size satisfies a bound of the form $h\beta \le C$ for a method-dependent constant $C$; in the SGN metric this yields the step-size ranges summarized in the following CFL-style policy.
Here $\alpha$ is any certified contraction margin (e.g., from SGN or the directional refinement of \cref{thm:dsc}) and $\beta$ is a Lipschitz bound for $G=-F$ on the region of interest in the $M(w)$-geometry; $C_4$ and $c_4$ are RK4-specific constants from \cref{thm:rk4} (in our experiments we used numerically validated choices such as $C_4\approx 2.5$, $c_4\approx \tfrac{1}{2}$).

\begin{center}
\fbox{\parbox{0.95\linewidth}{\textbf{Projected integrators in the SGN metric $\norm{\cdot}_{M(w)}$.}
All statements hold on any region $\mathcal{S}$ on which the corresponding SGN/DSC and Lipschitz bounds are certified and which is forward-invariant for the method (or is enforced by projection).
\begin{align*}
  &\text{Projected Euler: } &&x^{k+1} = \Pi_{\X}^{M(w)}\big(x^k + \eta\,G(x^k)\big),\quad 0<\eta<\tfrac{2\alpha}{\beta^2},\\
  &&&\hspace{4.0em}\text{one-step factor } \sqrt{1-2\alpha\eta+\beta^2\eta^2}\text{ (\cref{thm:proj-euler}).}\\[0.3em]
  &\text{RK4: } &&x^{k+1} = \Pi_{\X}^{M(w)}\!\big(\Phi_h(x^k)\big),\quad
      0<h\le \tfrac{C_4}{\beta},\\
  &&&\hspace{4.0em}\text{one-step factor } \exp(-c_4\alpha h)\text{ (\cref{thm:rk4}).}
\end{align*}
\vspace{-0.3em}}}
\end{center}

\begin{remark}[Projection on product sets]
For product strategy sets $\X=\prod \X_i$, the metric projection $\Pi_{\X}^{M(w)}$ decouples into independent per-player projections $\Pi_{\X_i}^{P_i}$ that are invariant under the scalar weighting $w$.
Thus, for the standard case of product constraints, the SGN update is simply a projected gradient step with a scalar step size $\eta$, and the weights $w$ serve purely as an analysis instrument to certify the contraction of this scalar-step dynamics where Euclidean analysis would fail.
\end{remark}

\begin{remark}[RK4 on constrained sets]
For constrained domains $\X$, we apply the metric projection $\Pi_{\X}^{M(w)}$ after each RK4 step, i.e., we use the one-step map $\mathcal{T}_h(x)=\Pi_{\X}^{M(w)}\!\big(\Phi_h(x)\big)$.
Here $\Phi_h$ is the unconstrained RK4 update for $\dot x = G(x)$; its intermediate stages require that $G$ be well-defined on a neighbourhood of $\X$.
Nonexpansiveness of $\Pi_{\X}^{M(w)}$ in $\norm{\cdot}_{M(w)}$ ensures that the contraction factor $\exp(-c_4\alpha h)$ from \cref{thm:rk4} is preserved for $\mathcal{T}_h$.
Projection generally reduces the local order of accuracy at the boundary (from fourth to second order), but our focus is on geometric convergence to an equilibrium rather than high-order trajectory accuracy, so this loss of order is immaterial for the present purposes.
\end{remark}

\section{Canonical Example: Non-Monotone Yet SGN-Amenable}

We illustrate the SGN geometry on a simple two-player quadratic game that is non-monotone in Euclidean geometry but admits a stabilizing weighted metric.

\subsection{Two-player SGN region and timescale band}

In the two-player case, the SGN condition admits a closed-form characterization: for each margin $\alpha<\min\{\mu_1,\mu_2\}$ there exist explicit roots $r_{-}(\alpha)<r_{+}(\alpha)$ such that $\mathrm{SGN}(\mu,L;P,w)$ with margin $\alpha$ holds if and only if the timescale ratio $r=w_2/w_1$ lies in the band $r\in(r_{-}(\alpha),r_{+}(\alpha))$; see Proposition~\ref{prop:two-player-band} in \cref{app:two-player-band} for the algebraic derivation.

\begin{remark}[Finite vs.\ asymptotic timescales and TTUR]
The two-player timescale band has a natural connection to two-timescale update rules (TTUR) for stochastic approximation, such as the TTUR analysis of GAN training~\citep{heusel2017}.
Classical TTUR results guarantee convergence by sending the timescale ratio $r_k=\eta_{2,k}/\eta_{1,k}\to\infty$, corresponding to infinite timescale separation between the ``slow'' and ``fast'' players.
Our two-player SGN band instead identifies a finite range of geometric ratios $r=w_2/w_1$ for which the existing scalar-step dynamics is contractive in the SGN metric.
In this sense SGN provides a non-asymptotic complement to TTUR: it does not alter the update rule, but it certifies when a finite timescale ratio suffices, giving a ``license to speed'' in games that Euclidean analysis would incorrectly classify as unstable.
Operationally, the SGN band for $r=w_2/w_1$ identifies the range of metric weights under which the scalar-step pseudo-gradient flow is contractive.
\end{remark}

\begin{remark}[High-dimensional timescale cones]
In the $N=2$ case, the SGN condition reduces to a scalar interval for $r=w_2/w_1$. For $N>2$, the linearity of $C(w,\alpha)$ in $w$ implies that the set of valid weights $\mathcal{W}=\{w\in\R_{++}^N:\ C(w,\alpha)\succ 0\}$ forms a convex cone in the positive orthant. While this does not admit a simple 1D interval representation, it geometrically generalizes the band concept: the ``safe'' timescales form a connected convex region where player update rates are balanced relative to the aggregate coupling structure. Operationally, identifying this cone (or optimizing the margin within it) remains efficient as a Generalized Eigenvalue Problem (GEVP).
\end{remark}

Consider now the quadratic example
\[
  f_1(x_1,x_2) = \tfrac{\mu_1}{2} x_1^2 + a\, x_1 x_2,\qquad
  f_2(x_1,x_2) = \tfrac{\mu_2}{2} x_2^2 + b\, x_1 x_2,
\]
so the pseudo-gradient is
\[
  F(x) = \begin{bmatrix} \mu_1 x_1 + a x_2\\ \mu_2 x_2 + b x_1 \end{bmatrix},
  \qquad
  J = \nabla F = \begin{bmatrix} \mu_1 & a \\ b & \mu_2 \end{bmatrix}.
\]

\paragraph{Euclidean non-monotonicity.}
In Euclidean geometry, $F$ is monotone iff the symmetric part $J_s=(J+J^\top)/2$ is positive semidefinite, i.e.,
$\mu_1\mu_2 \ge ((a+b)/2)^2$.
Choosing, for instance, $\mu_1=\mu_2=1$, $a=10$, $b=0.05$, we have $((a+b)/2)^2\approx 25.25>1=\mu_1\mu_2$, so $J_s$ is indefinite and Euclidean monotonicity fails.
Simultaneous gradient descent in this Euclidean geometry exhibits pronounced transient growth and oscillations, and diverges once the step-size exceeds the usual stability threshold.

\paragraph{SGN geometry.}
Take $P_1=P_2=1$, so $P=I_2$.
Block bounds are given by $\mu_1,\mu_2$ and
\[
  L_{12}=|a|=10,\qquad
  L_{21}=|b|=0.05.
\]
The SGN matrix $C(w,\alpha)$ depends on weights $w=(w_1,w_2)$ and margin $\alpha$.
In this $2\times 2$ case we have
\[
  C(w,\alpha)
  = \begin{bmatrix}
      2w_1(\mu_1-\alpha) & -(w_1L_{12}+w_2L_{21})\\
      -(w_1L_{12}+w_2L_{21}) & 2w_2(\mu_2-\alpha)
    \end{bmatrix}.
\]
Positive definiteness is equivalent to $w_i(\mu_i-\alpha)>0$ and
\[
  4w_1w_2(\mu_1-\alpha)(\mu_2-\alpha) - (w_1L_{12}+w_2L_{21})^2 > 0.
\]
Choosing weights that balance the off-diagonal terms, e.g.
$
  w_1/w_2 = L_{21}/L_{12} = 0.05/10 = 1/200,
$
so $r=w_2/w_1=200$,
gives $w_1L_{12}=w_2L_{21}$ and simplifies the determinant condition to
\[
  (\mu_1-\alpha)(\mu_2-\alpha) > L_{12}L_{21}.
\]
For the numerical values above, this reduces to $(1-\alpha)^2>0.5$, so any $\alpha\in(0,\,1-\sqrt{0.5})$ is admissible; thus SGN holds with a positive margin for the balanced weights.
In the corresponding $M(w)$-geometry, the vector field is less skewed and trajectories straighten, revealing a regime where SGN holds and contraction can be certified.
This illustrates how SGN leverages a weighted geometry to compensate for strong cross couplings; \cref{fig:quadratic-phase,fig:quadratic-pseudospectrum} visualize the resulting flow and pseudospectra in Euclidean vs SGN geometries.

\begin{figure}[htbp]
  \centering
  \includegraphics[width=0.6\linewidth]{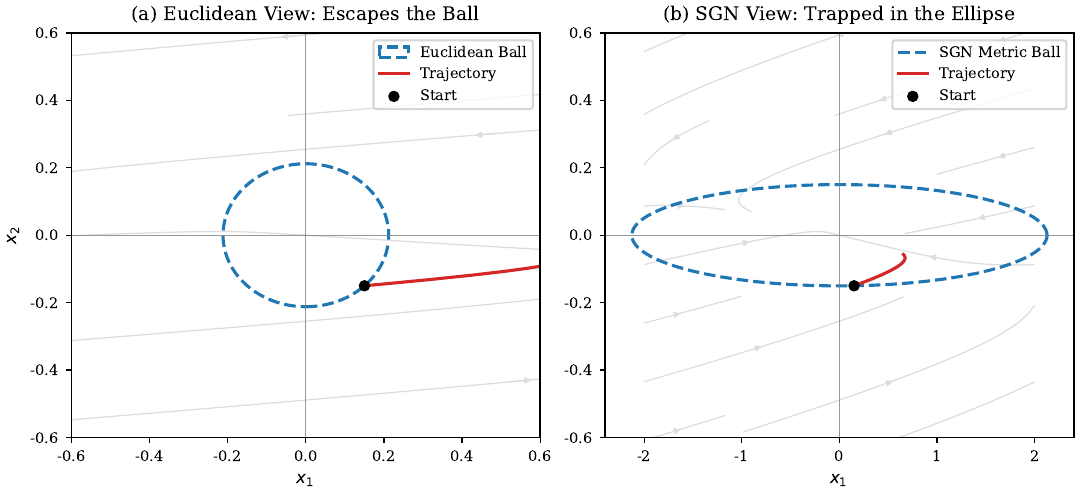}
  \caption{The ``Escape vs. Trap'' visualization of the quadratic example. Both panels display the \emph{same} standard gradient flow trajectory (red). Left: In the Euclidean view, the trajectory transiently escapes the unit circle (dashed blue) outwards, indicating a failure of monotonicity ($\frac{d}{dt}\|x\|^2 > 0$). Right: In the SGN-metric view (displayed as an ellipse in the original coordinates), the same trajectory strictly enters the sublevel set, visually confirming the contraction certificate provided by the weighted metric.}
  \label{fig:quadratic-phase}
\end{figure}

\begin{figure}[htbp]
  \centering
  \includegraphics[width=0.6\linewidth]{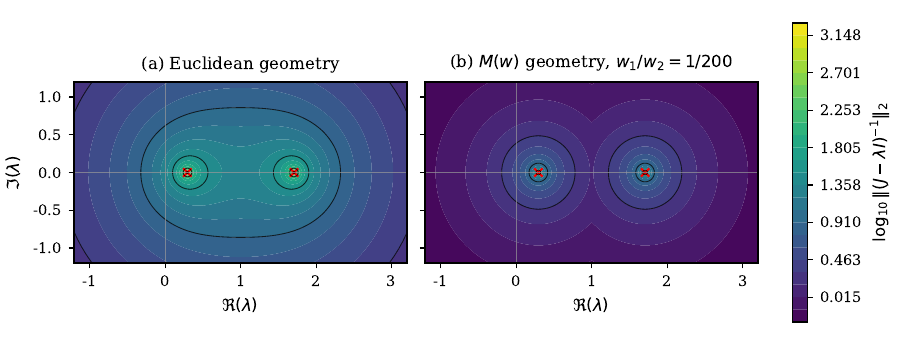}
  \caption{Pseudospectra of the quadratic example: in Euclidean geometry the pseudospectral bulge extends far beyond the eigenvalues, indicating strong nonnormal amplification, while under the SGN metric the pseudospectrum is much tighter, quantifying how the SGN geometry reduces nonnormal amplification.}
  \label{fig:quadratic-pseudospectrum}
\end{figure}

\paragraph{Discrete-time behavior.}
For the linear system $G(x)=-Jx$, a Lipschitz bound in any norm is given by an operator norm of $J$.
For example, in Euclidean geometry one has
$\beta\le \norm{J}_2 \le \norm{J}_F = \sqrt{\mu_1^2 + \mu_2^2 + a^2 + b^2}\approx 10.1$.
Given a certified monotonicity or DSC margin $\alpha>0$ in a metric $M(w)$ (e.g., on a region where the SGN condition holds), the RK4 policy of \cref{thm:rk4} yields a concrete safe step
$h\le C_4/\beta$.
Within this step range, the one-step map is a contraction in $\norm{\cdot}_{M(w)}$, providing a priori stability guarantee with per-step factor $\exp(-c_4\alpha h)$.
Numerically, for the balanced SGN metric and a step satisfying this bound, RK4 converges linearly to the origin, whereas using the same step in the Euclidean geometry leads to large oscillations or divergence, mirroring the qualitative behavior in \cref{fig:quadratic-phase}.

\FloatBarrier

\subsection{High-dimensional LQ validation}
\label{subsec:lq-validation}

The two-dimensional quadratic example above provides a visual illustration of the SGN geometry.
We now instantiate the same ideas on a higher-dimensional linear-quadratic (LQ) game where the block bounds $(\mu,L)$ used by SGN are \emph{exact}, and we compare the resulting SGN margin, Lipschitz constant, and discrete-time step-size bounds to the true behavior.

\paragraph{Canonical $64$-dimensional LQ game.}
For each coupling strength $\lambda\ge 0$, we consider a two-player game on $\R^{64}$ with blocks $d_1=d_2=32$ and costs
\[
  f_1(x_1,x_2) = \tfrac12 x_1^\top Q_1 x_1 + x_1^\top A_{12}(\lambda) x_2,\qquad
  f_2(x_1,x_2) = \tfrac12 x_2^\top Q_2 x_2 + x_2^\top A_{21}(\lambda) x_1,
\]
so that $F(x)=H(\lambda)x$ with
\[
  H(\lambda)
  =
  \begin{bmatrix}
    Q_1 & A_{12}(\lambda)\\
    A_{21}(\lambda) & Q_2
  \end{bmatrix}.
\]
We fix isotropic curvature $Q_i=\mu_0 I_{32}$ with $\mu_0=1$, so the block curvature parameters are exactly $\mu_1=\mu_2=\mu_0$.
For the cross couplings we choose an orthogonal base matrix $R\in\R^{32\times 32}$ with $\|R\|_2=1$ and set
\[
  A_{12}(\lambda) = \lambda a R,\qquad
  A_{21}(\lambda) = \lambda b R^\top,
\]
with $a=10$ and $b=0.05$ as in the scalar example.
For this canonical game the SGN block parameters
\[
  L_{12}(\lambda) = \|A_{12}(\lambda)\|_2 = \lambda a,\qquad
  L_{21}(\lambda) = \|A_{21}(\lambda)\|_2 = \lambda b
\]
are therefore equal to the true operator norms of the cross blocks, so the SGN matrix $C(w,\alpha)$ uses \emph{tight} curvature and coupling bounds.

\paragraph{Margins and step-sizes across $\lambda$.}
We equip $\R^{64}$ with the block-diagonal metric $M(w)=\mathrm{diag}(w_1 I_{32},w_2 I_{32})$ and, for a grid of couplings $\lambda\in[0,2.5]$, compute:
the Euclidean symmetric margin
\[
  \gamma_{\mathrm{euc}}(\lambda)
  := \lambda_{\min}\!\Big(\tfrac12\big(H(\lambda)+H(\lambda)^\top\big)\Big),
\]
the SGN margin $\alpha_*(w,\lambda)$ from the $2\times 2$ small-gain matrix $C(w,0)$, and the true metric margin
\[
  \alpha_{\mathrm{true}}(w,\lambda)
  := \lambda_{\min}\big(M(w)^{-1/2} J_M M(w)^{-1/2}\big),
  \quad
  J_M := \tfrac12\big(M(w)H(\lambda)+H(\lambda)^\top M(w)\big).
\]
We also compute the Lipschitz constant
\[
  \beta(w,\lambda) = \big\|M(w)^{1/2} H(\lambda) M(w)^{-1/2}\big\|_2
\]
of $G=-F$ in the $M(w)$-norm.
For consistency with the scalar example we take the balanced weights so that $w_2/w_1=L_{12}/L_{21}=a/b=200$.

For this canonical game the SGN bound and the true metric margin coincide over the entire regime where SGN certifies strong monotonicity: numerically we find $\alpha_*(w,\lambda)=\alpha_{\mathrm{true}}(w,\lambda)$ wherever $\alpha_*(w,\lambda)>0$.
This equality is specific to the structured choice of couplings; in more generic LQ games SGN provides a conservative lower bound on the true margin (see Appendix~\ref{app:lq-exp}).
\Cref{fig:lq-margins} summarizes how the Euclidean and SGN/true metric margins vary with the coupling $\lambda$, highlighting the SGN-only regime; see its caption for the detailed comparison.

\begin{figure}[htbp]
  \centering
  \includegraphics[width=0.6\linewidth]{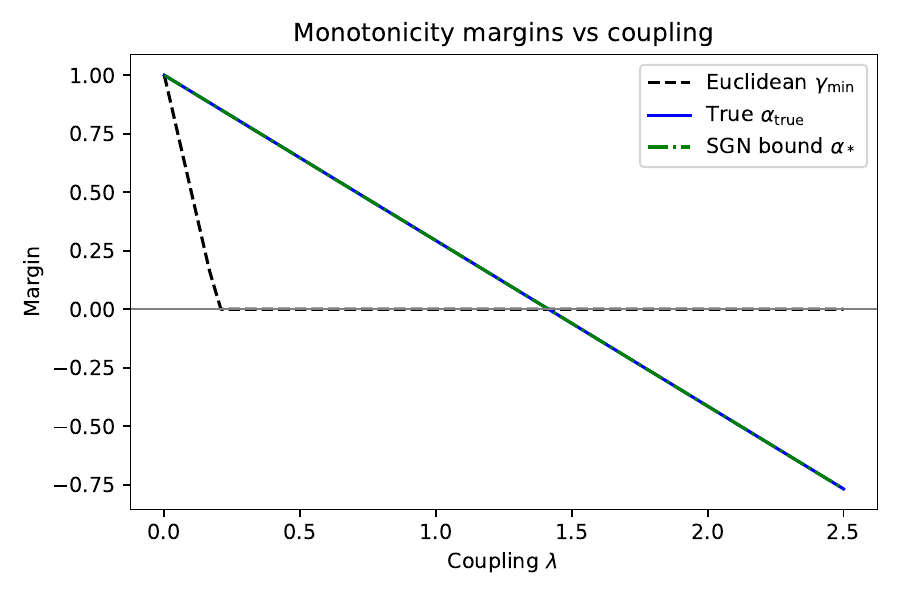}
  \caption{Monotonicity margins for the canonical $64$-dimensional LQ game: the Euclidean symmetric margin (dashed) becomes negative as the coupling~$\lambda$ increases, while the true metric margin (solid) and the SGN bound (dash-dotted) remain positive over an SGN-only regime.}
  \label{fig:lq-margins}
\end{figure}

\paragraph{Representative coupling and numerical values.}
We concentrate on the coupling value $\lambda_\star=1$, which lies inside this SGN-only regime.
At $\lambda_\star$ the metrics and bounds take particularly simple values for the balanced weights:
\begin{equation*}
  \gamma_{\mathrm{euc}}(\lambda_\star) \approx -4.03,\qquad
  \alpha_*(w,\lambda_\star) = \alpha_{\mathrm{true}}(w,\lambda_\star) \approx 0.293,\qquad
  \beta(w,\lambda_\star) \approx 1.71.
\end{equation*}
The Euler and RK4 SGN step-size bounds from \cref{thm:proj-euler,thm:rk4} are then
\[
  0 < \eta < \eta_{\mathrm{SGN}}(\lambda_\star)
  := \frac{2\alpha_*(w,\lambda_\star)}{\beta(w,\lambda_\star)^2}
  \approx 0.20,\qquad
  0 < h \le h_{\mathrm{SGN}}(\lambda_\star)
  := \frac{C_4}{\beta(w,\lambda_\star)}
  \approx 1.46,
\]
where we use a conservative choice $C_4\approx 2.5$ for the RK4 log-norm constant.
\Cref{tab:lq-representative} summarizes these quantities.

\begin{table}[t]
  \centering
  \caption{Euclidean and SGN margins and SGN step-size bounds for the canonical $64$-dimensional LQ game at the representative coupling $\lambda_\star=1$ with balanced weights $w_1/w_2=b/a$.
  The SGN margin coincides with the true metric margin for this structured game, while the Euclidean symmetric part has a large negative eigenvalue.}
  \label{tab:lq-representative}
  \begin{tabular}{lcc}
    \toprule
    Quantity & Expression & Value at $\lambda_\star$ \\
    \midrule
    Euclidean margin & $\gamma_{\mathrm{euc}}(\lambda_\star)$ & $\approx -4.03$ \\
    SGN margin & $\alpha_*(w,\lambda_\star)$ & $\approx 0.293$ \\
    True metric margin & $\alpha_{\mathrm{true}}(w,\lambda_\star)$ & $\approx 0.293$ \\
    Lipschitz constant & $\beta(w,\lambda_\star)$ & $\approx 1.71$ \\
    Euler SGN bound & $\eta_{\mathrm{SGN}}(\lambda_\star)$ & $\approx 0.20$ \\
    RK4 SGN bound & $h_{\mathrm{SGN}}(\lambda_\star)$ & $\approx 1.46$ \\
    \bottomrule
  \end{tabular}
\end{table}

\paragraph{Timescale band and weight design.}
The two-player SGN timescale band of \cref{prop:two-player-band} predicts a range of weight ratios $r=w_2/w_1$ for which $C(w,0)\succ 0$.
In the canonical game this band can be compared directly to the true metric margin $\alpha_{\mathrm{true}}(w,\lambda_\star)$, since the block bounds are tight.
\Cref{fig:lq-timescale} compares the analytic band with the region where $\alpha_{\mathrm{true}}(w,\lambda_\star)>0$, confirming that the small-gain condition is not an artifact of low dimension (see the caption for quantitative detail).

\begin{figure}[htbp]
  \centering
  \includegraphics[width=0.6\linewidth]{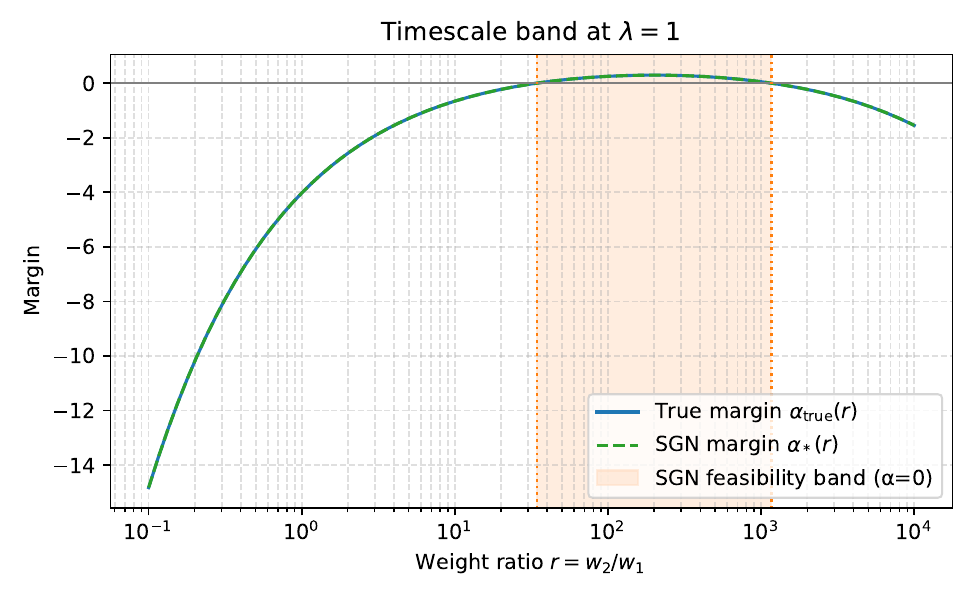}
  \caption{Timescale band for the canonical $64$-dimensional LQ game at $\lambda=1$: SGN and true metric margins across weight ratios $r=w_2/w_1$, together with the analytic SGN feasibility band, which closely matches the region where the true metric margin is positive.}
  \label{fig:lq-timescale}
\end{figure}

We note that linear-quadratic dynamics of the form studied here also arise in the NTK limit of wide neural networks, where training can be modeled as gradient flow in function space with a fixed kernel operator~\citep{jacot2018ntk}.
In that regime, SGN applies to the linearized game defined by the NTK and the loss Hessian, but we leave a detailed NTK-based treatment to future work.

\FloatBarrier

\paragraph{Discrete-time phase diagrams.}
Using the exact Jacobian $H(\lambda)$ and metric $M(w_{\mathrm{bal}})$ we compute, for each $\lambda$ on a grid, the SGN Euler and RK4 step-size bounds and the true stability thresholds obtained from spectral radii of the one-step matrices.
The resulting phase diagrams, shown in \cref{fig:lq-phase}, lie strictly inside the true stability region but remain within a modest factor of the actual thresholds, validating the contraction step-size policies in a high-dimensional setting.
In the coupling-stepsize $(\lambda,h)$ plane, the SGN step-size policies (orange dashed curves) stay strictly within the true stability boundaries (red solid curves) across all couplings, yet provide practical, nonvanishing steps that closely track the empirical stability frontier.

\begin{figure}[htbp]
  \centering
  \includegraphics[width=0.48\linewidth]{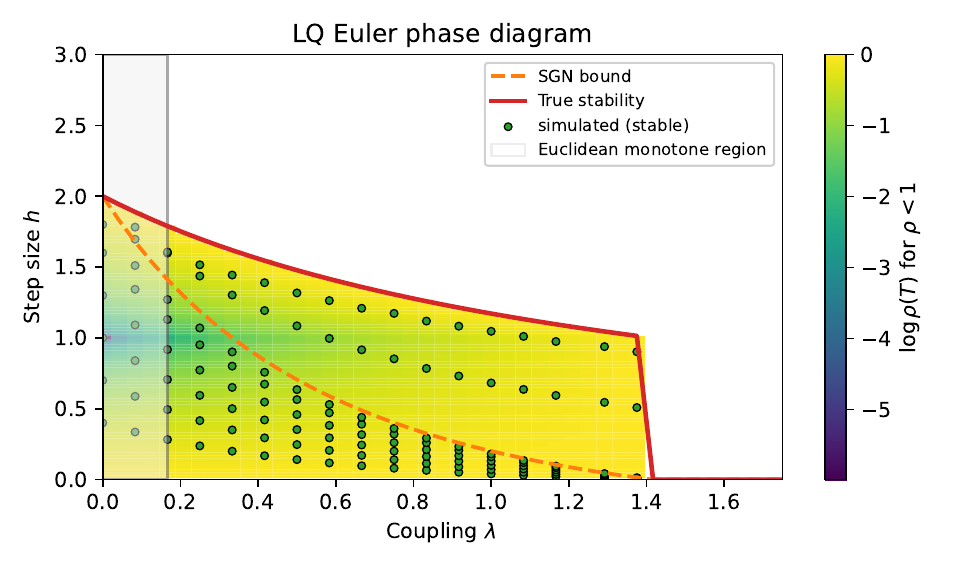}\hfill
  \includegraphics[width=0.48\linewidth]{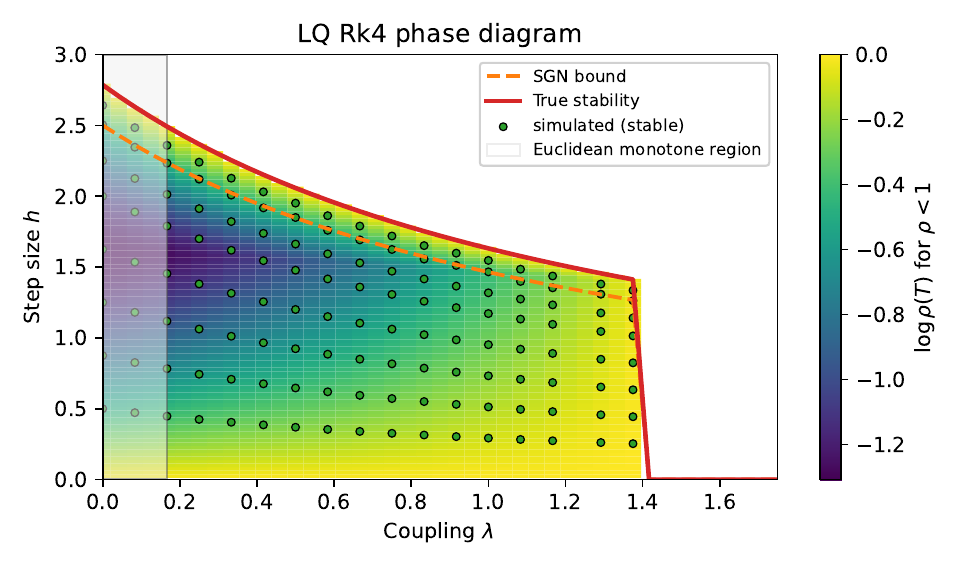}
  \caption{Discrete-time stability phase diagrams for the canonical LQ game in the balanced SGN metric.
  Left: projected Euler; right: RK4.
  The heat map shows $\log\rho(T(\lambda,h))$ for the one-step matrix (stable when $\log\rho<0$), and the orange dashed curve is the certified SGN step-size bound, which lies strictly inside the true stability boundary (red solid curve) across all coupling strengths $\lambda$ while remaining within a moderate factor of the empirical stability threshold.}
  \label{fig:lq-phase}
\end{figure}

\FloatBarrier

Additional numerical results for this LQ game---including continuous-time flow norms, structured coupling-noise robustness, and a random LQ ensemble---are collected in Appendix~\ref{app:lq-exp}.
These plots show that SGN-based margins closely track the true metric margins wherever certification succeeds, become gradually more conservative as the couplings deviate from the structured canonical case, and still provide nontrivial RK4 step-size bounds across a generic ensemble of random LQ games.

\section{Verifying SGN on Compact Regions}

The SGN condition is phrased in terms of block curvature and Lipschitz parameters $(\mu,L)$ and a Lipschitz constant $\beta$ for $G=-F$.
We briefly outline a certification pipeline showing that these quantities can be bounded on compact regions of interest.
In practice, the pipeline is most natural for games with modest effective dimension or exploitable structure (such as the quadratic and tabular Markov examples below, or latent-space or layer-wise subsystems of larger models), rather than full parameter spaces of very large deep networks.

\subsection{Certification pipeline and guarantees}

Let $\mathcal{R}\subseteq\X$ be a compact region, interpreted as a trust region or a forward-invariant set candidate.
In many applications it is natural to take $\mathcal{R}$ itself to be a sublevel set of the SGN Lyapunov function, e.g., an $M(w)$-ball $B_r(x^*)$ or a more general ellipsoidal set $\{x:\,V(x)\le c\}$ around a reference point $x^*$; for such sets, contraction of the flow or one-step map immediately yields forward invariance as in \cref{cor:ball-forward}.
Given a block metric $P=\mathrm{diag}(P_i)$, the pipeline proceeds as follows:
\begin{enumerate}
  \item \textbf{Curvature and couplings.}
  Estimate block curvature and cross-player couplings $(\mu_i^{\mathrm{lo}},L_{ij}^{\mathrm{hi}})$ in the $P$-geometry on $\mathcal{R}$ so that the block bounds~\eqref{eq:block-bounds} hold there with $(\mu,L)=(\mu^{\mathrm{lo}},L^{\mathrm{hi}})$.
  \item \textbf{Local Lipschitz constant.}
  For a weight vector $w\in\R_{++}^N$ and associated metric $M(w)$, bound Jacobian norms of $G=-F$ in the $M(w)$-geometry on $\mathcal{R}$ to obtain a local Lipschitz constant $\beta^{\mathrm{hi}}(w)$ obeying~\eqref{eq:G-Lip}.
  \item \textbf{Small-gain margin and geometry design.}
  With $(\mu^{\mathrm{lo}},L^{\mathrm{hi}})$ fixed, search over $w\in\R_{++}^N$ and margins $\alpha>0$ such that the small-gain matrix $C(w,\alpha)$ from~\eqref{eq:Cwalpha} is positive definite, defining an SGN margin $\alpha_*(w)$ and optionally choosing $w$ to trade off $\alpha_*(w)$ against $\beta^{\mathrm{hi}}(w)$.
  \item \textbf{Certificate and step-sizes.}
  For a chosen $w$, output the metric $M(w)$, margin $\alpha:=\alpha_*(w)$, and local Lipschitz bound $\beta^{\mathrm{hi}}(w)$, which together (via \cref{lem:strong-mono,thm:proj-euler,thm:rk4}) yield a contraction certificate and safe step-size ranges for projected Euler and RK4 on $\mathcal{R}$.
\end{enumerate}

Implementation details for curvature/coupling and Lipschitz estimation (analytic bounds versus sampling and spectral probes) are collected in Appendix~\ref{app:sgn-cert}.

\subsection{Operational usage}

We emphasize that certification is intended as an \emph{offline} or slowly updated design-time procedure.
Once a region $\mathcal{R}$, metric $M(w)$, margin $\alpha$, and Lipschitz bound $\beta^{\mathrm{hi}}$ have been certified, one can run projected Euler or RK4 with the corresponding step-size policies
\[
  0<\eta<\frac{2\alpha}{(\beta^{\mathrm{hi}})^2},\qquad
  0<h\le \frac{C_4}{\beta^{\mathrm{hi}}}.
\]
If the system's operational regime changes slowly, one may periodically re-certify on updated regions, separating the heavier offline certification cost from the lightweight online iteration cost.

\section{Extensions: Mirror Geometry and Markov Games}

We now illustrate how the SGN viewpoint extends beyond fixed Euclidean block metrics.

\subsection{Mirror SGN on the simplex}
\label{subsec:mirror-sgn}

In games with mixed strategies, each player's action lies on a probability simplex $\Delta_{m_i}$ or, more precisely, on its relative interior $\Delta_{m_i}^\circ=\{x_i\in\R^{m_i}:\ x_{i,k}>0,\ \sum_k x_{i,k}=1\}$ equipped with the Euclidean structure induced by the affine hyperplane $\{\sum_k x_{i,k}=1\}$.
The SGN framework is geometry-agnostic: block curvature and coupling parameters $(\mu_i,L_{ij})$ can be defined using local mirror metrics and associated Bregman divergences instead of Euclidean norms~\citep{nemirovski1983}.
In Appendix~\ref{app:mirror-sgn}, we formalize a mirror-SGN condition on such simplices and show that, under the same block bounds and a mirror version of $\mathrm{SGN}(w,\alpha)$, the mirror pseudo-gradient flow is exponentially contractive with Lyapunov function
$
  V(x)=\sum_i w_i D_{\psi_i}(x_i^*\Vert x_i)
$
and converges to an interior Nash equilibrium (Theorem~\ref{thm:mirror-flow}), while mirror-Euler and mirror-RK schemes inherit geometric decay of $V$ (Corollary~\ref{cor:mirror-euler-rk}).
Choosing entropic mirror maps on simplices recovers contraction guarantees in the natural Fisher/information geometry for entropy-regularized policy-gradient dynamics.

\subsection{Markov game experiment: mirror-SGN for natural policy gradient}

We next summarize a complementary \emph{pseudo-gradient} validation in a small tabular Markov game that instantiates the mirror-SGN geometry and the discrete-time step-size bounds.

\paragraph{Setup.}
We consider a discounted two-player Markov game~\citep{puterman1994} with state space $\mathcal{S}=\{s_0,s_1\}$, binary actions $\mathcal{A}_i=\{0,1\}$, and discount factor $\gamma=0.9$.
At each state $s$, both players receive a cooperative reward that is $+1$ when they coordinate ($00$ or $11$) and $-1$ on mismatches; the transition kernel encourages staying in the current state under consistent coordination and flips the state under the opposite coordinated action or mismatches.
We choose a minimal tabular setting to allow exact visualization of the Fisher-SGN geometry, which would be obscured in high-dimensional approximations.
Each player~$i$ uses a tabular softmax policy $\pi_i(\cdot\mid s)$ and an entropy-regularized objective
\[
  J_i(\pi)
  = \mathbb{E}_\pi\Big[\sum_{t=0}^\infty \gamma^t r(s_t,a_t)\Big]
    + \tau \sum_{s\in\mathcal{S}} H\big(\pi_i(\cdot\mid s)\big),
  \qquad f_i = -J_i,
\]
with negative entropy $H$ and regularization strength $\tau=1$.
The mirror map is therefore the negative entropy on $\Delta_2\times\Delta_2$ and the associated local metric is the Fisher information matrix of the softmax parametrization, matching the mirror-SGN construction of Appendix~\ref{app:mirror-sgn}; we denote the player-wise Fisher blocks by $\mathcal{F}_i(\theta)$.
In these Fisher coordinates the pseudo-gradient $F(\theta)=(\nabla_{\theta_1} f_1,\nabla_{\theta_2} f_2)$ is well-defined, and the natural policy gradient update
\[
  \theta_i^{k+1}
  = \theta_i^k - \eta\,\mathcal{F}_i(\theta^k)^{-1}\,\nabla_{\theta_i} f_i(\theta^k)
\]
is the standard natural policy gradient update (with scalar step size $\eta$), which we analyze using the weighted Fisher metric $M(w)=\mathrm{diag}\bigl(w_1 \mathcal{F}_1(\theta),w_2 \mathcal{F}_2(\theta)\bigr)$ to certify contraction.

\paragraph{Local mirror-SGN certificate.}
At the unique interior Nash equilibrium $\theta^*=0$ (uniform policies), we certify a small neighbourhood in logit space by gridding a cube
$\{\theta:\,\|\theta-\theta^*\|_\infty \le 0.1\}$ and evaluating the Fisher-normalized curvature, couplings, and Jacobian at each grid point (using the sampling-based SGN pipeline of Appendix~\ref{app:sgn-cert}).
For every grid point $\theta$ we form the own and cross Hessian blocks of $f_i$ and rescale them as in Appendix~\ref{app:mirror-sgn} to obtain local $(\mu_i(\theta),L_{ij}(\theta))$.
Taking the minimum and maximum over this grid yields conservative block bounds
$\mu_1^{\mathrm{lo}}\approx\mu_2^{\mathrm{lo}}\approx 0.93$ and
$L_{12}^{\mathrm{hi}}\approx L_{21}^{\mathrm{hi}}\approx 0.60$ for the mirror geometry around $\theta^*$, so the mirror-SGN two-player matrix $H(w)$ admits a positive margin for a nontrivial range of weight ratios $r=w_2/w_1$.
Optimizing over $r$ with these bounds gives a best diagonal margin $\alpha_*(w_\star)\approx 0.33$ at a near-balanced ratio $r_\star\approx 0.98$.
Using the same Fisher metric and grid we then bound the Jacobian log-norm of the natural-gradient field in the metric $M(w_\star)$, obtaining a local Lipschitz bound $\beta\approx 1.57$.
By the mirror-SGN contraction theorem (Appendix~\ref{app:mirror-sgn}) and the discrete-time Euler bound of \cref{thm:proj-euler}, any step size
\[
  0 < \eta < \eta_{\mathrm{SGN}}
  := \frac{2\alpha_*(w_\star)}{\beta^2}
  \approx 0.27
\]
ensures that the mirror Lyapunov function
$V(\theta)=\sum_i w_i D_{\psi_i}(x_i^*\Vert x_i(\theta))$ contracts along the natural policy gradient iterates in a neighbourhood of $\theta^*$ at a rate predicted by the factor $\sqrt{1-2\alpha_*\eta+\beta^2\eta^2}$.
In particular, in this tabular Markov game, NPG enjoys a certified local linear convergence guarantee to the interior Nash equilibrium, with contraction rate and admissible step-size band determined explicitly by $(\alpha_*(w_\star),\beta)$ and the Fisher-SGN metric $M(w_\star)$.

\paragraph{Natural vs Euclidean policy gradient.}
We instantiate the certified dynamics at a conservative step $\eta=0.5\,\eta_{\mathrm{SGN}}$ and compare:
(i) natural policy gradient (NPG) in the Fisher metric $M(w_\star)$, and
(ii) Euclidean policy gradient (EPG) on the logit parameters with the same scalar step size $\eta$.
For each method we initialize policies near uniform and track the mirror Lyapunov $V_k$ and the $\ell_2$ distance to $\theta^*$.
As shown in \cref{fig:markov-lyap}, NPG exhibits clean exponential decay of $\log V_k$ at a rate substantially faster than EPG, and the empirical slope lies safely below the SGN upper bound implied by $(\alpha_*,\beta)$; EPG also converges in this benign setting, but with a significantly slower decay and no comparable stepsize certificate.
The right panel of \cref{fig:markov-lyap} sweeps the NPG and EPG step sizes over multiples of $\eta_{\mathrm{SGN}}$ and shows that NPG is empirically stable (convergent across seeds) for $\eta/\eta_{\mathrm{SGN}}\lesssim 1$, with increasingly frequent failures beyond the SGN line, while EPG has a different and method-dependent stability profile.
This matches the intended role of SGN as a \emph{sufficient} design-time bound for NPG rather than a sharp, method-independent phase transition.

\begin{figure}[htbp]
  \centering
  \includegraphics[width=0.48\linewidth]{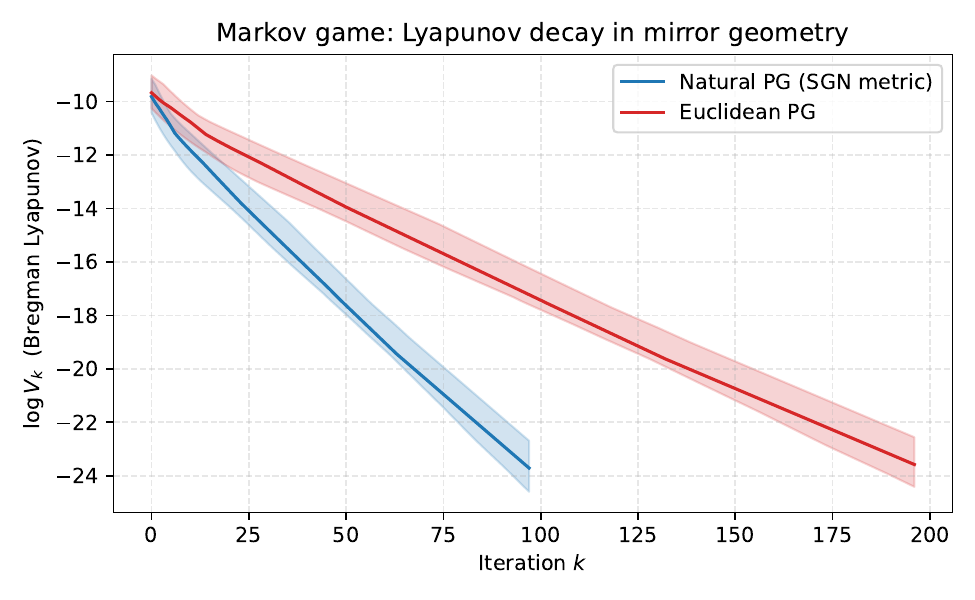}\hfill
  \includegraphics[width=0.48\linewidth]{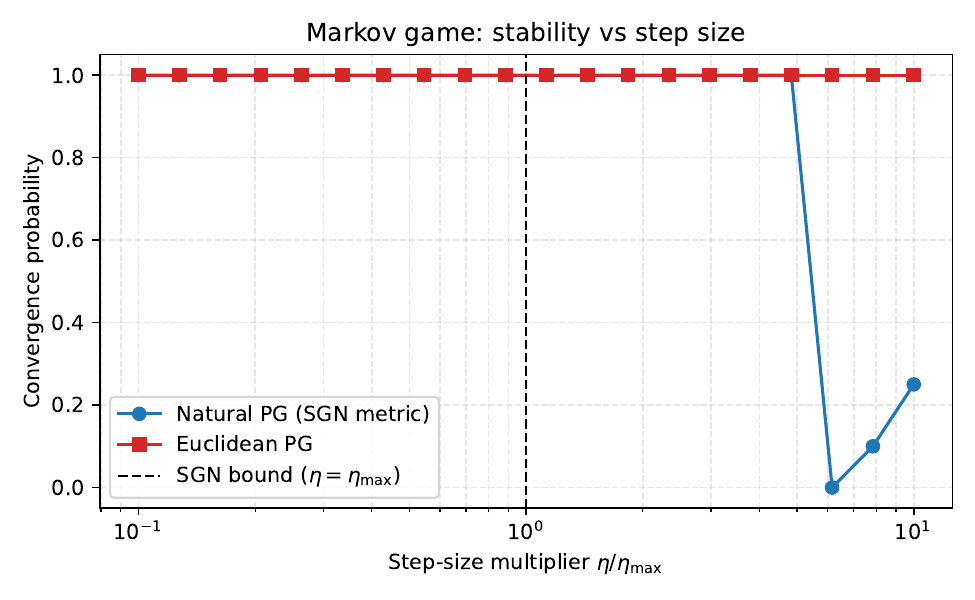}
  \caption{Mirror Lyapunov decay and stability for natural (blue) vs Euclidean (red) policy gradient in the tabular Markov game. Left: natural policy gradient in the Fisher metric exhibits clean exponential decay of the Lyapunov function, substantially faster than Euclidean policy gradient. Right: convergence probability (fraction of seeds classified as convergent---either triggering the small-gradient stopping criterion or ending with $\|\theta_k-\theta^*\|_2 < 1.0$ after 500 steps) as a function of the step-size multiplier $\eta/\eta_{\mathrm{SGN}}$, showing that NPG remains stable up to roughly the SGN bound while EPG has a different stability profile.}
  \label{fig:markov-lyap}
\end{figure}

\FloatBarrier

\paragraph{Markov timescale band.}
Finally, we examine the effect of relative player timescales by varying the weight ratio $r=w_2/w_1$ in the Fisher metric while keeping the game and regularization fixed.
Using the same locally estimated $(\mu_i,L_{ij})$ we compute the mirror-SGN margin $\alpha_*(r)$ from the $2\times 2$ matrix $H(w)$ and obtain a nontrivial band of ratios where $\alpha_*(r)>0$; \cref{fig:markov-timescale} plots this band.
The band has the same qualitative shape as in the LQ game (\cref{fig:lq-timescale}): extreme timescale imbalances violate the small-gain inequalities, while a moderate interval of ratios admits a diagonal mirror metric that certifies local contraction.
This Markov example therefore confirms that the SGN timescale intuition and metric-design perspective extend beyond static quadratic games to entropy-regularized policy gradients in tabular Markov games.

\begin{figure}[htbp]
  \centering
  \includegraphics[width=0.6\linewidth]{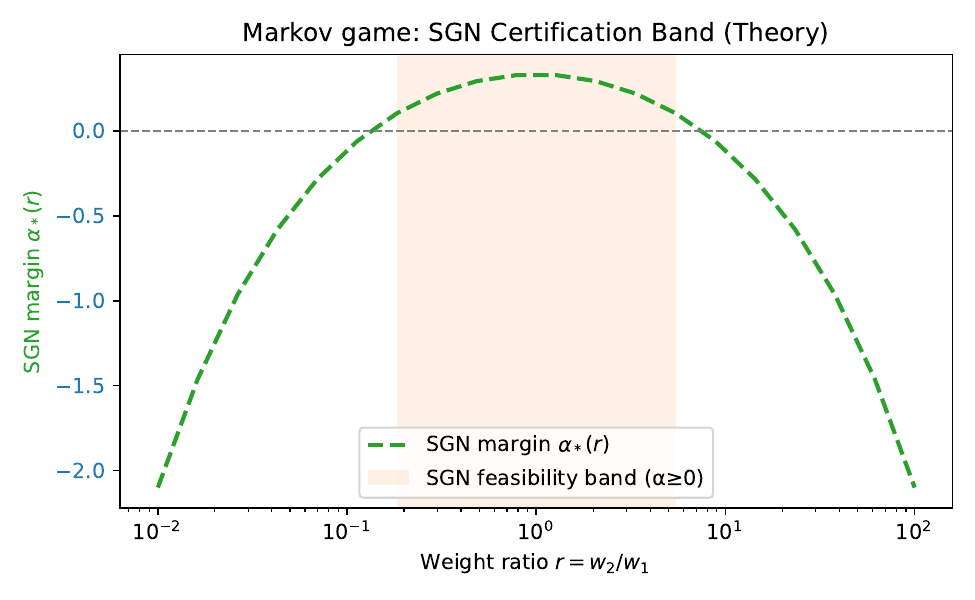}
  \caption{Mirror-SGN timescale band for the tabular Markov game: SGN margin as a function of the weight ratio $r=w_2/w_1$ in the Fisher metric, highlighting a nontrivial interval of ratios that admit a diagonal mirror metric certifying local contraction.}
  \label{fig:markov-timescale}
\end{figure}

\FloatBarrier

\section{Future Work}
\label{sec:future-work}

An immediate direction is to relax the convexity assumptions to systematically estimate basins of attraction in nonconvex landscapes, moving beyond our current certification of fixed compact trust regions to the analysis of bifurcation phenomena and multi-equilibrium stability.
A further step is to extend the mirror/Fisher SGN framework beyond finite-dimensional simplices to more general statistical manifolds and function-approximation policies, where the underlying Fisher metrics live on infinite-dimensional spaces of probability distributions.

Our tabular Markov game example and mirror/Fisher extension highlight a closely related but distinct direction: modeling reinforcement learning algorithms (such as actor-critic) as explicit two-block SGN systems on a statistical manifold, with a policy block in Fisher geometry and an environment/critic block capturing Bellman dynamics.
This ``Fisher-Bellman small-gain'' viewpoint would rigorously decouple policy entropy (curvature) from the value-function scale and discount horizon (coupling), potentially yielding sharper stability conditions for actor-critic methods than those obtained by treating the Bellman operator as a generic Lipschitz disturbance.
Pursuing this model could lead to small-gain conditions that explicitly balance entropy-induced curvature against reward scale and discount horizon; we leave a full development of such Fisher-Bellman SGN models to future work.

Beyond existence of a suitable metric, one can treat SGN as a \emph{metric design} problem: optimize over diagonal or block metrics to improve a contraction margin relative to a Lipschitz or log-norm bound (e.g., by minimizing a matrix measure of $J_G$ or a condition proxy such as $\beta/\alpha_*$), and embed this into an IQC/dissipativity framework~\citep{megretski1997,boyd1994,lessard2016} that co-designs both geometry and algorithm parameters, potentially via vector-dissipativity viewpoints where each player contributes its own Lyapunov component.
Finally, systematic empirical evaluation remains important: implementing the certification pipeline on structured quadratic, Markov, and moderately sized differentiable games; comparing certified step sizes and rates with empirically stable regimes; and studying how the cost and conservatism of SGN certificates scale with dimension and number of players.
Such studies would indicate which game classes are most amenable to SGN-based certification and where additional structural assumptions or approximations are needed.

\section{Related Work}

Convergence of gradient-based dynamics to Nash equilibria under monotonicity conditions has a long history, starting with Rosen's diagonal strict concavity~\citep{rosen1965} and the variational inequality framework~\citep{facchinei2003}.
Rosen's condition can be interpreted as the existence of a positive diagonal scaling that renders the symmetric part of the pseudo-gradient positive definite; in matrix terms, this is closely related to diagonal stability conditions of the form $DJ + J^\top D\succ 0$ for some $D\succ 0$~\citep[see, e.g., diagonal stability surveys in control and network theory]{boyd1994}.
Our SGN condition is a block-diagonal, small-gain refinement of these ideas: instead of arbitrary scalings we work with coarse block summaries $(\mu,L)$ and a structured weight vector $w$, and we phrase the diagonal dominance condition in terms of the small-gain matrix $C(w,\alpha)$, which plays the role of a block Gershgorin/diagonal-stability certificate for the pseudo-gradient in the metric $M(w)$.

A closely related line of work applies contraction theory directly to Nash equilibrium seeking. \citet{gokhale2023contractivity} derive gain-matrix conditions for strong contractivity of pseudo-gradient and best-response dynamics in weighted Euclidean seminorms. In their setting, each player's dynamics are strongly contracting in an individual norm, and the cross-couplings enter through a Metzler gain matrix $\Gamma$; Hurwitzness of $\Gamma$, together with a network contraction theorem, implies the existence of a weighted norm under which the joint Nash dynamics are contracting. In the Euclidean, single-block case our SGN condition induces a gain matrix with the same qualitative structure, and the requirement $C(w,\alpha)\succ 0$ can be read as a block-diagonal refinement of these Metzler-matrix conditions.

Small-gain theorems originate in robust and nonlinear control~\citep{zames1966,jiang1994} and provide interconnection stability under gain composition conditions.
The normalized gain matrix $H(w)$ and the SGN matrix $C(w,\alpha)$ can be viewed as input-output gain composition inequalities for the players' subsystems, with $\mu_i$ and $L_{ij}$ acting as self- and cross-gains; the SGN feasibility region is then directly analogous to classical small-gain inequalities phrased in terms of spectral radii or diagonal scalings.

Passivity-based methods in games~\citep{feijer2010,lessard2016} yield convergence under energy dissipation conditions, while potential games~\citep{monderer1996} admit convergence via scalar potential descent.
More recent notions such as variational stability and coherence~\citep{mertikopoulos2018} relax monotonicity via curvature or averaging assumptions.
Our work complements these by providing a block-structured sufficient condition, expressed in terms of curvature and Lipschitz bounds, together with an explicit metric-design viewpoint that selects a block-diagonal geometry to recover strong monotonicity on a certified region.

On the contraction side, our DSC condition is a constant-metric instance of the general contraction-metric framework in nonlinear control and dynamical systems, where one searches for a (possibly state-dependent) Riemannian metric $M(x)$ that renders a flow contractive in differential coordinates~\citep{lohmiller1998,forni2014}.
In full generality the contraction inequality must account for the Levi-Civita connection and the Lie derivative of $M(x)$, whereas SGN deliberately restricts attention to constant, block-diagonal metrics $M(w)$ and expresses contractivity via algebraic small-gain inequalities in terms of block curvature and coupling parameters $(\mu_i,L_{ij})$.
On regions where SGN holds, the resulting metric $M(w)$ induces a simple quadratic (or Bregman, in the mirror case) Lyapunov function, so the condition can be read as identifying a class of games that are provably contractive together with their explicit Lyapunov certificates.

On the time-discretization side, explicit integrators for contractive systems have been extensively studied in numerical analysis~\citep{hairer1993,butcher2003,soderlind2006,trefethen2005}.
We build directly on this literature: the log-norm/matrix-measure bounds and stability regions for explicit Euler and RK methods in general norms are classical, and our contribution is to instantiate these results in the SGN metric and to translate them into concrete step-size policies for projected Euler and RK4 in the geometries induced by $M(w)$ and its mirror variants.

Algorithms such as Extragradient~\citep{korpelevich1976} and OGDA~\citep{popov1980,mertikopoulos2018} improve convergence in monotone and certain non-monotone settings.
SGN is complementary to these methods: it provides a structural condition and a geometry in which both simple methods (Euler, RK) and more sophisticated ones (EG, OGDA) can be analyzed, and could in principle be combined with them by applying SGN-based metric design to their underlying pseudo-gradient operators.

\section{Conclusion}

We introduced a block-structured small-gain condition, Small-Gain Nash, that certifies strong monotonicity of the pseudo-gradient in a tailored block geometry.
Under SGN, the continuous pseudo-gradient flow and simple explicit discretizations are contracting (on the region where SGN and Lipschitz bounds hold) with explicit step-size policies.
The framework extends to a local mirror/Fisher geometry on simplices---where Bregman Lyapunov functions and mirror flow/discretization contraction guarantees can be obtained---and we have outlined how the same small-gain ideas can be employed in closed-loop Markov games.
We also described a certification pipeline that, on compact regions, turns the structural SGN condition into a formal and computable convergence certificate.

Our viewpoint is closer to control and numerical analysis than to benchmark optimization: the goal is not a new training algorithm, but a framework for \emph{certifying} that existing dynamics will converge in a given game.
We believe that such structural and geometric tools are essential for bringing rigorous guarantees to multi-agent and other feedback-coupled learning and control systems where stability is non-negotiable.

\paragraph{Code availability.}
Code reproducing all experiments is available at:\\
\url{https://github.com/AashVed/SmallGainNash}.

\bibliographystyle{abbrvnat}
\bibliography{references}

\appendix

\section{Additional Details on Verifying Small-Gain}
\label{app:sgn-cert}

We briefly describe how the bounds $(\mu^{\mathrm{lo}},L^{\mathrm{hi}},\beta^{\mathrm{hi}})$ in the certification pipeline can be obtained in practice on a compact region $\mathcal{R}$, and collect proofs of the core continuous- and discrete-time contraction statements used in the main text.

\paragraph{Tier 1: analytic bounds.}
If the $f_i$ are given by structured expressions (e.g., quadratic forms, compositions of layers with known spectral bounds), one can propagate bounds on Hessian blocks and Jacobian norms across the computational graph using standard rules:
positive semidefinite lower bounds for own-player Hessians and operator-norm upper bounds for cross-player blocks.
This yields conservative but inexpensive estimates of $\mu_i^{\mathrm{lo}}$ and $L_{ij}^{\mathrm{hi}}$, as well as a bound $\beta^{\mathrm{hi}}$ on $\norm{J_G(x)}_{M(w)\to M(w)}$ over $\mathcal{R}$.

\paragraph{Tier 2: refinement via sampling.}
To reduce conservatism, one can sample $\mathcal{R}$ on a finite cover $S$ (e.g., a low-discrepancy design or a grid in low dimension) and, at each $s\in S$, perform:
\begin{itemize}
  \item power iteration on block Jacobians to estimate $\norm{J_{ij}(s)}_{P_j\to P_i}$ for each $(i,j)$, updating $L_{ij}^{\mathrm{hi}}$ to the maximum over samples;
  \item Lanczos iterations on $\nabla^2_{x_i x_i} f_i(s)$ to estimate the smallest eigenvalue in the $P_i$-metric, updating $\mu_i^{\mathrm{lo}}$ to the minimum over samples;
  \item power iteration on $J_G(s)$ in the $M(w)$-metric to refine $\beta^{\mathrm{hi}}$.
\end{itemize}
If the Hessian and Jacobian vary Lipschitzly on $\mathcal{R}$, these sample-wise estimates can be inflated by Lipschitz bounds and the cover radius to yield global bounds on $\mathcal{R}$ in the $M(w)$-geometry.
In our experiments we use the sampling-based estimates directly without interval inflation, so the resulting certificates are numerically validated but not worst-case rigorous in the strict sense.

\paragraph{From bounds to SGN and steps.}
Given $(\mu^{\mathrm{lo}},L^{\mathrm{hi}},\beta^{\mathrm{hi}})$ on $\mathcal{R}$, one then searches for weights $w\in\R_{++}^N$ and a margin $\alpha>0$ such that $C(w,\alpha)\succ 0$ with entries defined by~\eqref{eq:Cwalpha}, using, e.g., a line search over $\alpha$ and an eigenvalue check for positive definiteness.
Any such $(w,\alpha)$ yields a certified SGN margin and metric $M(w)$ on $\mathcal{R}$.
When Jacobian information is available, one can optionally refine the margin by estimating a directional/DSC constant $\alpha_{\mathrm{dsc}}\ge\alpha$ from the $M(w)$-symmetric Jacobian $S_{M(w)}(x)$ via the condition~\eqref{eq:dsc} and replace $\alpha$ by $\max\{\alpha,\alpha_{\mathrm{dsc}}\}$.
By \cref{lem:strong-mono,thm:dsc,thm:proj-euler,thm:rk4} and standard results on strongly monotone flows, any resulting $\alpha$ and $\beta^{\mathrm{hi}}$ produce explicit step-size ranges for projected Euler and RK4 that guarantee contraction on $\mathcal{R}$.

\subsection{Gershgorin-style sufficient conditions for SGN}
\label{subsec:normalized-gain}

The matrix condition $C(w,\alpha)\succ 0$ is the exact SGN certificate.
Here we record a conservative but easy-to-check alternative that depends only on curvature, couplings, and the weight ratios $w_i/w_j$.
This condition implies SGN (and hence strong monotonicity) but need not be necessary.

On a region $\mathcal{R}$ where the block bounds~\eqref{eq:block-bounds} hold, introduce the nonnegative scalars
\[
  a_i := \sqrt{w_i}\,\norm{x_i-y_i}_{P_i},\qquad i=1,\dots,N,
\]
and coefficients
\[
  k_{ij} := L_{ij}\sqrt{\frac{w_i}{w_j}}\qquad(i\neq j).
\]
Define a symmetric matrix $H(w)\in\R^{N\times N}$ by
\begin{equation}
  H_{ii}(w) := \mu_i,\qquad
  H_{ij}(w) := -\,\tfrac{1}{2}\bigl(k_{ij}+k_{ji}\bigr)\quad(i\neq j).
  \label{eq:H-w-def}
\end{equation}
Combining the block bounds with these definitions yields the quadratic inequality
\begin{equation}
  \ip{x-y}{F(x)-F(y)}_{M(w)}
  \;\ge\; a^\top H(w)\,a
  \qquad\text{for all }x,y\in\mathcal{R}.
  \label{eq:H-lower-bound}
\end{equation}

If $H(w)\succeq\alpha I$ for some $\alpha>0$, then $a^\top H(w)a\ge \alpha\norm{a}_2^2$ for all $a$.
Moreover,
\[
  \norm{a}_2^2 = \sum_i w_i \norm{x_i-y_i}_{P_i}^2 = \norm{x-y}_{M(w)}^2.
\]
Thus from~\eqref{eq:H-lower-bound} we obtain the strong monotonicity inequality
\begin{equation}
  \ip{x-y}{F(x)-F(y)}_{M(w)}
  \;\ge\; \alpha\,\norm{x-y}_{M(w)}^2
  \qquad\forall x,y\in\mathcal{R}.
  \label{eq:H-strong-mono}
\end{equation}
In particular, $F$ is $\alpha$-strongly monotone in the metric $M(w)$, and the conclusions of \cref{lem:strong-mono,cor:unique-nash} and the exponential convergence of the pseudo-gradient flow hold with this margin.

One convenient way to certify $H(w)\succeq\alpha I$ is via strict diagonal dominance with a margin.
This yields the following constructive sufficient condition.

\begin{theorem}[Gershgorin-type sufficient condition]
\label{thm:gershgorin-normalized}
Let $\mu_{\min}:=\min_i \mu_i$.
Suppose there exist weights $w_i>0$ and a number $\alpha\in(0,\mu_{\min})$ such that, for every $i$,
\begin{equation}
  \mu_i - \alpha
  \;>\;
  \frac{1}{2}\sum_{j\neq i}\Bigl(
    L_{ij}\sqrt{\frac{w_i}{w_j}}
    + L_{ji}\sqrt{\frac{w_j}{w_i}}
  \Bigr).
  \label{eq:normalized-gershgorin}
\end{equation}
Then $H(w)\succeq\alpha I$, and hence $F$ is $\alpha$-strongly monotone in the metric $M(w)$ on $\mathcal{R}$.
\end{theorem}

\begin{proof}[Proof of \cref{thm:gershgorin-normalized}]
By construction~\eqref{eq:H-w-def}, all off-diagonal entries of $H(w)$ are nonpositive.
Condition~\eqref{eq:normalized-gershgorin} states that the diagonal entries of $H(w)-\alpha I$ are strictly larger than the sum of the absolute values of the off-diagonal entries in the corresponding rows.
By Gershgorin's circle theorem, all eigenvalues of $H(w)-\alpha I$ are strictly positive, so $H(w)\succeq\alpha I$.
Substituting this into~\eqref{eq:H-lower-bound} and using $\norm{a}_2^2=\norm{x-y}_{M(w)}^2$ yields~\eqref{eq:H-strong-mono}, and the final claim follows from the discussion in \cref{subsec:normalized-gain}.
\end{proof}

The condition~\eqref{eq:normalized-gershgorin} depends on $w$ only through the ratios $w_i/w_j$.
In the simplest case $w_i\equiv 1$ one obtains a ``Rosen-type'' Euclidean sufficient condition.
Optimizing over $w$ allows one to reduce the right-hand side of~\eqref{eq:normalized-gershgorin} and enlarge the admissible margin~$\alpha$.
It is often convenient to express~\eqref{eq:normalized-gershgorin} in terms of a normalized gain matrix.
For example, fixing reference weights $\bar w_i>0$ and defining
\[
  K_{ij} := \frac{L_{ij}}{\mu_i},\quad K_{ii}:=0,
\]
classical small-gain and diagonal-scaling arguments (see, e.g.,~\citep{zames1966,jiang1994,boyd1994}) imply that the spectral radius $\rho(K)$ controls the existence of diagonal scalings that satisfy inequalities of the form~\eqref{eq:normalized-gershgorin}.
If $\rho(K)<1$, then there exists a positive vector $v$ and a constant $\lambda\in(0,1)$ such that
\[
  \sum_{j\neq i} K_{ij} v_j \;\le\; \lambda v_i,\qquad i=1,\dots,N,
\]
and any such $v$ can be converted into a choice of weights $w$ guaranteeing~\eqref{eq:normalized-gershgorin} with a margin $\alpha$ proportional to $(1-\lambda)$.
The Perron-Frobenius eigenvector of $K$ corresponds to the least conservative such scaling.
We do not pursue a full spectral characterization here and instead use~\eqref{eq:normalized-gershgorin} as a practical sufficient condition that can be evaluated in $O(N^2)$ time independently of the parameter dimension.

\subsection{Two-player SGN band}
\label{app:two-player-band}

We record here a closed-form characterization of the two-player timescale band used in the main text.
Consider two players with curvature parameters $\mu_1,\mu_2>0$ and cross-Lipschitz constants $L_{12},L_{21}\ge 0$.
For any weights $w_1,w_2>0$ and any margin $\alpha\in\R$, the SGN matrix specializes to
\[
  C(w,\alpha)
  =
  \begin{bmatrix}
    2w_1(\mu_1-\alpha) & -(w_1L_{12}+w_2L_{21})\\
    -(w_1L_{12}+w_2L_{21}) & 2w_2(\mu_2-\alpha)
  \end{bmatrix}.
\]
Positive definiteness $C(w,\alpha)\succ 0$ is equivalent to $C_{11}>0$, $C_{22}>0$, and $\det C>0$.
The diagonal conditions simply require
\[
  \mu_1>\alpha,\qquad \mu_2>\alpha,
\]
and do not constrain the ratio $w_2/w_1$ as long as $w_i>0$.
The nontrivial constraint comes from the determinant:
\[
  \det C(w,\alpha)
  =
  4w_1w_2(\mu_1-\alpha)(\mu_2-\alpha) - (w_1L_{12}+w_2L_{21})^2 > 0.
\]
Introducing the ratio
\[
  r := \frac{w_2}{w_1} > 0
\]
and dividing by $w_1^2$ gives
\begin{equation}
  4(\mu_1-\alpha)(\mu_2-\alpha)\,r
  \;-\; (L_{12}+rL_{21})^2
  \;>\; 0.
  \label{eq:two-player-det-ineq}
\end{equation}
Expanding \eqref{eq:two-player-det-ineq} yields a strict quadratic inequality in $r$.
Real roots exist if and only if
\begin{equation}
  (\mu_1-\alpha)(\mu_2-\alpha) > L_{12}L_{21},
  \label{eq:two-player-condition}
\end{equation}
which is the two-player analogue of Rosen's diagonal strict concavity condition (in the symmetric case $L_{12}=L_{21}$).
When \eqref{eq:two-player-condition} holds, solving this quadratic yields two real roots
\begin{equation}
  r_{\pm}(\alpha)
  :=
  \frac{2(\mu_1-\alpha)(\mu_2-\alpha) - L_{12}L_{21}
        \pm
        2\sqrt{(\mu_1-\alpha)(\mu_2-\alpha)\big((\mu_1-\alpha)(\mu_2-\alpha)-L_{12}L_{21}\big)}}
       {L_{21}^2},
  \label{eq:two-player-roots}
\end{equation}
and, since the leading coefficient is negative, the quadratic is positive on the open interval $(r_{-}(\alpha),r_{+}(\alpha))$ and negative outside it.
Consequently, the determinant condition is satisfied if and only if
\begin{equation}
  r_{-}(\alpha) < r < r_{+}(\alpha).
  \label{eq:two-player-band}
\end{equation}

\begin{proposition}[Two-player SGN timescale band]
\label{prop:two-player-band}
Fix $\alpha<\min\{\mu_1,\mu_2\}$ and assume \eqref{eq:two-player-condition} holds.
Then for any weights $w_1,w_2>0$ the SGN matrix $C(w,\alpha)$ is positive definite if and only if the ratio
\[
  r = \frac{w_2}{w_1}
\]
lies strictly between the two roots $r_{-}(\alpha)$ and $r_{+}(\alpha)$ defined in~\eqref{eq:two-player-roots}.
In particular, the set
\[
  \mathcal{W}_\alpha
  :=
  \bigl\{(w_1,w_2)\in\R_{++}^2:\ r_{-}(\alpha)<w_2/w_1<r_{+}(\alpha)\bigr\}
\]
is exactly the set of weight pairs for which $\mathrm{SGN}(\mu,L;P,w)$ holds with margin~$\alpha$.
The weight ratio $r$ therefore parameterizes a ``safe timescale band'' for the two players in the SGN geometry.
\end{proposition}

\begin{proof}
The diagonal conditions $C_{11}>0$ and $C_{22}>0$ are equivalent to $\mu_i>\alpha$ and do not restrict $r$.
The determinant condition $\det C(w,\alpha)>0$ is equivalent to the quadratic inequality~\eqref{eq:two-player-det-ineq} in $r$.
Under~\eqref{eq:two-player-condition}, this quadratic has two distinct real roots given by~\eqref{eq:two-player-roots}; since the leading coefficient is negative, the quadratic is positive exactly when $r\in(r_{-}(\alpha),r_{+}(\alpha))$, which yields~\eqref{eq:two-player-band} and the description of $\mathcal{W}_\alpha$.
\end{proof}

For $\alpha=0$, condition~\eqref{eq:two-player-condition} reduces to $\mu_1\mu_2>L_{12}L_{21}$ and the band~\eqref{eq:two-player-roots} simplifies to
\[
  r_{\pm}(0)
  =
  \frac{2\mu_1\mu_2 - L_{12}L_{21}
        \pm
        2\sqrt{\mu_1\mu_2(\mu_1\mu_2-L_{12}L_{21})}}
       {L_{21}^2},
\]
which is a simple closed-form expression for a safe timescale ratio band.

\subsection{Proofs of core SGN and DSC results}
\label{app:proofs}

We record here the proofs of the strong monotonicity lemma, the VI/Nash corollary, and the DSC refinement theorem from the main text.

\begin{proof}[Proof of \cref{lem:strong-mono}]
Write $e_i:=x_i-y_i$ and $a_i:=\norm{e_i}_{P_i}\ge 0$.
Multiplying \eqref{eq:block-bounds} by $w_i$ and summing over $i$ yields
\begin{align*}
  \sum_i w_i\ip{e_i}{\nabla_{x_i} f_i(x)-\nabla_{x_i} f_i(y)}_{P_i}
  &\ge \sum_i w_i \mu_i a_i^2
  \;-\; \sum_i\sum_{j\ne i} w_i L_{ij}\,a_i a_j.
\end{align*}
By symmetry and $L_{ij}\ge 0$,
\[
  \sum_i\sum_{j\ne i} w_i L_{ij}\,a_i a_j
  \;=\; \sum_{i<j} (w_i L_{ij} + w_j L_{ji})\,a_i a_j.
\]
Introduce the symmetric matrix $C(w,0)$ with entries
\[
  C_{ii}(w,0) = 2w_i\mu_i,\qquad
  C_{ij}(w,0) = -\big(w_i L_{ij}+w_j L_{ji}\big)\quad(i\ne j).
\]
Then the inequality above can be written as
\[
  \sum_i w_i\ip{e_i}{\nabla_{x_i} f_i(x)-\nabla_{x_i} f_i(y)}_{P_i}
  \;\ge\; \tfrac{1}{2} a^\top C(w,0) a,
\]
where $a=(a_1,\dots,a_N)^\top$.
Similarly,
\[
  \alpha\norm{x-y}_{M(w)}^2
  = \alpha \sum_i w_i a_i^2
  = \tfrac{1}{2} a^\top \big(2\alpha\,\mathrm{diag}(w_i)\big)a.
\]
Thus \eqref{eq:strong-mono} is equivalent to
\[
  a^\top\big(C(w,0)-2\alpha\,\mathrm{diag}(w_i)\big)a \;\ge\; 0
  \qquad\forall a\in\R^N.
\]
By definition,
$C(w,\alpha) = C(w,0)-2\alpha\,\mathrm{diag}(w_i)$.
The assumption $C(w,\alpha)\succ 0$ implies
$a^\top C(w,\alpha)a>0$ for all $a\ne 0$, which yields the desired inequality.
Finally,
\[
  \sum_i w_i\ip{e_i}{\nabla_{x_i} f_i(x)-\nabla_{x_i} f_i(y)}_{P_i}
  = \ip{x-y}{F(x)-F(y)}_{M(w)},\qquad
  \sum_i w_i a_i^2 = \norm{x-y}_{M(w)}^2,
\]
by the definition of $M(w)$.
\end{proof}

\begin{proof}[Proof of \cref{cor:unique-nash}]
By \cref{lem:strong-mono}, the pseudo-gradient $F$ is $\alpha$-strongly monotone in the inner product induced by $M(w)$.
Since $M(w)\succ 0$, this is equivalent, under the change of variables $z=M(w)^{1/2}x$, to $\alpha$-strong monotonicity of the transformed operator $\hat F(z):=M(w)^{-1/2}F(M(w)^{-1/2}z)$ in the Euclidean inner product on $\hat\X:=M(w)^{1/2}\X$.
Existence and uniqueness of the solution of $\mathrm{VI}(\hat F,\hat\X)$ therefore follow directly from standard monotone VI theory~\citep[Theorem~2.3.3]{facchinei2003}, and the equivalence of $\mathrm{VI}(\hat F,\hat\X)$ and $\mathrm{VI}(F,\X)$ under this linear change of variables yields a unique solution $x^*$ to $\mathrm{VI}(F,\X)$.
When each $\X_i$ is convex and $f_i(\cdot,x_{-i})$ is convex for all $i$, solutions of $\mathrm{VI}(F,\X)$ coincide with Nash equilibria on product convex sets~\citep[Section~1.4]{facchinei2003}, so the Nash equilibrium is unique.
\end{proof}

\paragraph{Directional stability and logarithmic norms.}
For an SPD metric $M\succ 0$ and Jacobian $J_G(x)$ of $G=-F$, define the $M$-symmetric part
\[
  S_M(x) := \tfrac{1}{2}\big(M J_G(x) + J_G(x)^\top M\big).
\]
We say the system is \emph{diagonally stable in $M$} on a set $\mathcal{R}$ with margin $\alpha>0$ if
\begin{equation}
  S_M(x) \;\preceq\; -\,\alpha\,M
  \qquad\text{for all }x\in\mathcal{R},
  \label{eq:dsc}
\end{equation}
the DSC condition used in the main text.
The associated logarithmic norm (matrix measure) induced by $\norm{\cdot}_M$ is
\[
  \mu_M(A)
  := \sup_{\norm{z}_M=1} \ip{z}{A z}_M
  = \sup_{\norm{z}_M=1} z^\top \tfrac{1}{2}\big(MA + A^\top M\big) z,
\]
so \eqref{eq:dsc} is equivalent to $\mu_M(J_G(x))\le -\alpha$ for all $x\in\mathcal{R}$, the standard non-Euclidean contractivity condition used in log-norm analyses of ODE integrators~\citep{soderlind2006,hairer1993,butcher2003}.

\begin{theorem}[Directional/DSC refinement implies contraction]
\label{thm:dsc}
Suppose \eqref{eq:dsc} holds on a convex set $\mathcal{R}$.
Then for all $x,y\in\mathcal{R}$,
\[
  \ip{x-y}{G(x)-G(y)}_{M} \;\le\; -\,\alpha\,\norm{x-y}_{M}^2.
\]
Consequently, the continuous flow $\dot x=G(x)$ is $\alpha$-contracting in $\norm{\cdot}_M$ on $\mathcal{R}$.
If, in addition, $G$ is $\beta$-Lipschitz in $\norm{\cdot}_M$ on $\mathcal{R}$, then the discrete-time contraction results of \cref{thm:proj-euler,thm:rk4} apply on $\mathcal{R}$ with this directional margin $\alpha$ in place of the SGN margin.
\end{theorem}

\begin{proof}[Proof of \cref{thm:dsc}]
Along the segment $x_t := y + t(x-y)$ we have $G(x)-G(y) = \int_0^1 J_G(x_t)\,(x-y)\,dt$, so
\[
  \ip{x-y}{G(x)-G(y)}_{M}
  = \int_0^1 (x-y)^\top S_M(x_t)\,(x-y)\,dt
  \;\le\; -\,\alpha\,\norm{x-y}_{M}^2,
\]
using \eqref{eq:dsc}.
Contraction of the flow with rate $\alpha$ in $\norm{\cdot}_M$ is standard from matrix-measure arguments, and the discrete-time statements follow by repeating the proofs of \cref{thm:proj-euler,thm:rk4} with the one-sided inequality above in place of strong monotonicity of $F$.
\end{proof}

\begin{remark}[Contraction metrics and Riemannian view]
The DSC condition~\eqref{eq:dsc} can be interpreted as a constant-metric instance of Riemannian contraction: $M$ plays the role of a (state-independent) Riemannian metric tensor and the inequality $S_M(x)\preceq-\alpha M$ ensures exponential decay of differential displacements in this metric along the flow~\citep[see, e.g.,][]{lohmiller1998,forni2014}.
More general contraction-metric frameworks allow state-dependent metrics $M(x)$, in which case the contraction inequality acquires additional terms involving the Levi-Civita connection and the Lie derivative of $M(x)$; by restricting to constant, block-diagonal metrics $M(w)$, SGN avoids these geometric terms and reduces the contraction condition to the algebraic Jacobian inequality~\eqref{eq:dsc}.
\end{remark}

\paragraph{SGN implies DSC.}
On a region where the block bounds \eqref{eq:block-bounds} hold and the weighted SGN condition of \cref{def:sgn} is satisfied with margin $\alpha>0$, \cref{lem:strong-mono} implies
\[
  \ip{\Delta}{F(x+\Delta)-F(x)}_{M(w)} \;\ge\; \alpha\,\norm{\Delta}_{M(w)}^2
  \qquad\forall x,\Delta.
\]
Letting $G=-F$ and using the differentiability of $F$ to differentiate this inequality at $\Delta=0$ shows that
$\ip{\Delta}{J_G(x)\,\Delta}_{M(w)} \le -\alpha\norm{\Delta}_{M(w)}^2$ for all $\Delta$ and all $x$ in the region.
By polarization, this is equivalent to $S_{M(w)}(x)\preceq -\alpha M(w)$, i.e., \eqref{eq:dsc} with $M=M(w)$.
Thus SGN produces a diagonally stable metric and guarantees at least the margin $\alpha$; Jacobian-based estimation of $S_{M(w)}(x)$ can only certify a directional margin from \eqref{eq:dsc} that matches or exceeds this SGN margin, so the SGN margin acts as a lower bound.

\section{Mirror Geometry and Small-Gain Nash}
\label{app:mirror-sgn}

Mirror/Fisher variants of our results use standard notions of mirror maps and Bregman divergences on the players' strategy spaces (for example, negative entropy and Kullback-Leibler divergence on simplices).
We sketch how the mirror-geometry ingredients underlying \cref{thm:mirror-flow} can be derived from Hessian bounds and transported to the Lyapunov analysis for mirror flows and mirror discretizations.
Throughout, we work on a compact region $\mathcal{R}$ contained in the interior of the players' domains where the Hessians of $f_i$ and $\psi_i$ are uniformly bounded and invertible.

Rather than imposing symmetric block bounds for all pairs $(x,y)$ as in \cref{def:block-bounds}, it is convenient in mirror geometry to work with inequalities relative to a fixed equilibrium $x^*$.
We orient the Bregman divergences as $D_{\psi_i}(x_i^*\Vert x_i)$ so that they can be used directly as components of a Lyapunov function measuring distance \emph{from} $x$ \emph{to} $x^*$.

\begin{definition}[Mirror block bounds around $x^*$]
Let $x^*\in\X$ be a Nash equilibrium and let $(\psi_i)_i$ be mirror maps.
We say that the pseudo-gradient $F$ obeys mirror block bounds $(\mu,L)$ around $x^*$ on a region $\mathcal{R}\subseteq\X$ if for all $x\in\mathcal{R}$ and each player $i$,
\begin{equation}
  \big\langle x_i - x_i^*,\,\nabla_{x_i} f_i(x) - \nabla_{x_i} f_i(x^*)\big\rangle
  \;\ge\;
  \mu_i\,D_{\psi_i}(x_i^*\Vert x_i)
  \;-\;
  \sum_{j\neq i} L_{ij}\,
    \sqrt{D_{\psi_i}(x_i^*\Vert x_i)}\,
    \sqrt{D_{\psi_j}(x_j^*\Vert x_j)}.
  \label{eq:mirror-block-bounds}
\end{equation}
Here $\mu_i>0$ and $L_{ij}\ge 0$ play the role of curvature and coupling parameters in the mirror geometry.
\end{definition}

Given mirror block bounds, we define a normalized gain matrix using the Bregman distances to $x^*$.
For weights $w\in\R_{++}^N$ define
\[
  a_i(x) := \sqrt{w_i\,D_{\psi_i}(x_i^*\Vert x_i)},\qquad
  k_{ij} := L_{ij}\sqrt{\frac{w_i}{w_j}},
\]
and let $H(w)\in\R^{N\times N}$ be the symmetric matrix with entries
\begin{equation}
  H_{ii}(w) := \mu_i,\qquad
  H_{ij}(w) := -\tfrac{1}{2}\bigl(k_{ij}+k_{ji}\bigr)\quad(i\neq j),
  \label{eq:mirror-H-w-def}
\end{equation}
exactly as in \cref{subsec:normalized-gain}.
We say that \emph{mirror-$\mathrm{SGN}(w,\alpha)$} holds with margin $\alpha>0$ if $H(w)\succeq \alpha I$.

Multiplying~\eqref{eq:mirror-block-bounds} by $w_i$ and summing over $i$ shows that
\[
  \big\langle x-x^*,\,F(x)-F(x^*)\big\rangle_w
  \;\ge\;
  a(x)^\top H(w)\,a(x),
\]
where $\langle u,v\rangle_w := \sum_i w_i\langle u_i,v_i\rangle$ and $a(x)=(a_i(x))_i$.
If $H(w)\succeq\alpha I$, then $a(x)^\top H(w)a(x)\ge \alpha\|a(x)\|_2^2 = \alpha\sum_i w_i D_{\psi_i}(x_i^*\Vert x_i)$.
We record this as a mirror strong-monotonicity property.

\begin{lemma}[Mirror strong monotonicity at $x^*$]
\label{lem:mirror-strong-mono}
Suppose mirror block bounds~\eqref{eq:mirror-block-bounds} hold on $\mathcal{R}$ and mirror-$\mathrm{SGN}(w,\alpha)$ holds for some $w\in\R_{++}^N$ and $\alpha>0$.
Then for all $x\in\mathcal{R}$,
\begin{equation}
  \big\langle x-x^*,\,F(x)-F(x^*)\big\rangle_w
  \;\ge\;
  \alpha \sum_{i=1}^N w_i\,D_{\psi_i}(x_i^*\Vert x_i).
  \label{eq:mirror-strong-mono}
\end{equation}
\end{lemma}

Next we consider a mirror analogue of the pseudo-gradient flow $\dot x=-F(x)$.
Introduce dual variables $z_i=\nabla\psi_i(x_i)$ and let $z=(z_1,\dots,z_N)$.
The mirror pseudo-gradient flow is defined in dual coordinates by
\begin{equation}
  \dot z_i(t) = G_i\big(x(t)\big) = -\,\nabla_{x_i} f_i\big(x(t)\big),\qquad
  x_i(t) = \nabla\psi_i^*(z_i(t)).
  \label{eq:mirror-flow}
\end{equation}
For negative entropy mirror maps on simplices this flow coincides with a natural-gradient ODE whose trajectories remain in the interior of the product simplex~\citep{amari1998,bauschke2011}.
Define the Bregman Lyapunov function
\[
  V(x) := \sum_{i=1}^N w_i\,D_{\psi_i}(x_i^*\Vert x_i).
\]

\begin{theorem}[Mirror flow contraction under mirror SGN]
\label{thm:mirror-flow}
Assume each $\psi_i$ is a Legendre mirror map on $\X_i$ and $x^*$ is an interior Nash equilibrium with $F(x^*)=0$ in the chosen coordinates (for example, after reparameterizing policies via logits so that the simplex constraint is implicit).
Suppose mirror block bounds~\eqref{eq:mirror-block-bounds} hold on a region $\mathcal{R}$ containing $x^*$ and mirror-$\mathrm{SGN}(w,\alpha)$ holds there for some $w\in\R_{++}^N$ and $\alpha>0$.
Let $x(t)$ solve the mirror flow~\eqref{eq:mirror-flow} with $x(0)\in\mathcal{R}$ and assume $x(t)\in\mathcal{R}$ for all $t\ge 0$.
Then $V$ decays exponentially along the flow:
\[
  \frac{d}{dt} V\big(x(t)\big)
  \;\le\; -\,\alpha\,V\big(x(t)\big),
  \qquad
  V\big(x(t)\big) \;\le\; e^{-\alpha t}\,V\big(x(0)\big),
\]
and in particular $x(t)\to x^*$ as $t\to\infty$.
\end{theorem}

\begin{corollary}[Mirror Euler and mirror RK contraction]
\label{cor:mirror-euler-rk}
Under the assumptions of \cref{thm:mirror-flow}, suppose in addition that $G$ satisfies Lipschitz and log-norm bounds on $\mathcal{R}$ analogous to those in \cref{thm:proj-euler,thm:rk4} in the weighted inner product $\langle\cdot,\cdot\rangle_w$.
Then there exist step-size ranges of the form
\[
  0<\eta<\frac{2\alpha}{\beta^2},\qquad
  0<h\le \frac{C_4}{\beta},
\]
for which mirror-Euler and mirror-RK4, respectively, are contractions with respect to the Bregman Lyapunov function $V$:
\[
  V(x^{k+1}) \;\le\; q\,V(x^k)
\]
for some $q\in(0,1)$ and all iterates that remain in $\mathcal{R}$.
\end{corollary}

\subsection{Mirror block bounds from Hessians}

Assume each $f_i$ is twice continuously differentiable on $\mathcal{R}$ and that there exist positive numbers $\mu_i$ and $L_{ij}$ such that, for all $x\in\mathcal{R}$,
\[
  \nabla_{x_i x_i}^2 f_i(x) \succeq \mu_i\,\nabla^2\psi_i(x_i),\qquad
  \big\|\nabla_{x_i x_j}^2 f_i(x)\big\|_{\psi_j\to\psi_i} \le L_{ij}\quad(j\neq i),
\]
where $\|\cdot\|_{\psi_j\to\psi_i}$ denotes the mixed operator norm induced by the local metrics $\nabla^2\psi_j$ and $\nabla^2\psi_i$.
Then a standard integration-along-segments argument analogous to \cref{def:block-bounds} shows that the mirror block inequality~\eqref{eq:mirror-block-bounds} holds on $\mathcal{R}$ with these $(\mu_i,L_{ij})$, with Bregman divergences $D_{\psi_i}(x_i^*\Vert x_i)$ playing the role of squared distances.

\subsection{Proof of mirror strong monotonicity}

Starting from~\eqref{eq:mirror-block-bounds}, multiply by $w_i>0$ and sum over players to obtain
\[
  \sum_i w_i\big\langle x_i-x_i^*,\,\nabla_{x_i} f_i(x) - \nabla_{x_i} f_i(x^*)\big\rangle
  \;\ge\;
  \sum_i w_i\mu_i D_{\psi_i}(x_i^*\Vert x_i)
  \;-\;
  \sum_{i\neq j} w_i L_{ij}
    \sqrt{D_{\psi_i}(x_i^*\Vert x_i)}\,
    \sqrt{D_{\psi_j}(x_j^*\Vert x_j)}.
\]
Introducing $a_i(x):=\sqrt{w_i D_{\psi_i}(x_i^*\Vert x_i)}$ and $k_{ij}:=L_{ij}\sqrt{w_i/w_j}$ yields
\[
  \big\langle x-x^*,\,F(x)-F(x^*)\big\rangle_w
  \;\ge\;
  a(x)^\top H(w)\,a(x),
\]
with $H(w)$ as in~\eqref{eq:mirror-H-w-def}.
If $H(w)\succeq \alpha I$, then $a(x)^\top H(w)a(x)\ge \alpha\|a(x)\|_2^2 = \alpha\sum_i w_i D_{\psi_i}(x_i^*\Vert x_i)$, which is exactly~\eqref{eq:mirror-strong-mono}.

\subsection{Lyapunov derivative for the mirror flow}

For Legendre mirror maps the change of variables $z_i=\nabla\psi_i(x_i)$ is a diffeomorphism between a primal neighbourhood of $x_i^*$ and a dual neighbourhood of $z_i^*:=\nabla\psi_i(x_i^*)$.
Recall the mirror flow~\eqref{eq:mirror-flow},
\[
  \dot z_i(t) = G_i\big(x(t)\big),\qquad
  x_i(t) = \nabla\psi_i^*(z_i(t)),
\]
with $G=-F$.
Using the representation
\[
  D_{\psi_i}(x_i^*\Vert x_i)
  = \psi_i(x_i^*) - \psi_i(x_i) - \big\langle\nabla\psi_i(x_i),\,x_i^*-x_i\big\rangle,
\]
and differentiating along a trajectory $x(t)$ shows that
\[
  \frac{d}{dt} D_{\psi_i}(x_i^*\Vert x_i(t))
  = \big\langle G_i\big(x(t)\big),\,x_i(t)-x_i^*\big\rangle,
\]
where we have used $\dot z_i = \nabla^2\psi_i(x_i)\dot x_i$ and the Legendre property to cancel the $\psi_i$ and $\langle z_i,x_i\rangle$ terms.
Summing over players with weights $w_i$ yields
\[
  \frac{d}{dt} V\big(x(t)\big)
  = \sum_i w_i \big\langle G_i(x(t)),\,x_i(t)-x_i^*\big\rangle
  = \big\langle x(t)-x^*,\,G(x(t)) - G(x^*)\big\rangle_w,
\]
since $G(x^*)=-F(x^*)=0$ at an interior Nash equilibrium.
Combining this with the mirror strong-monotonicity inequality~\eqref{eq:mirror-strong-mono} and $G=-F$ gives
\[
  \frac{d}{dt} V\big(x(t)\big)
  \le -\,\alpha\,V\big(x(t)\big),
\]
which proves \cref{thm:mirror-flow}.

\subsection{Mirror Euler and mirror RK discretizations}

Finally, the contraction of mirror-Euler and mirror-RK4 schemes follows by transporting the Euclidean arguments of \cref{thm:proj-euler,thm:rk4} to dual coordinates.
In dual space the mirror flow is an ODE $\dot z = G(\nabla\psi^*(z))$ endowed with the weighted Euclidean inner product $\langle\cdot,\cdot\rangle_w$.
The log-norm and Lipschitz bounds induced by mirror block bounds and mirror SGN yield one-sided inequalities for $\langle z-z^*,G(\nabla\psi^*(z)) - G(\nabla\psi^*(z^*))\rangle_w$ exactly as in the Euclidean case.
Explicit Euler and RK4 applied to this dual ODE therefore satisfy the same contractivity estimates as in \cref{thm:proj-euler,thm:rk4} for sufficiently small step sizes.
Numerically, one integrates the unconstrained dual dynamics in $\R^d$ (for example, Fisher-natural logits on the simplex) using standard RK4 stages and maps back to the primal variables via $x=\nabla\psi^*(z)$ only for analysis or termination, so intermediate RK4 stages never hit the primal boundary and no projection on the simplex is required.
When the dual domain itself is constrained, metric nonexpansiveness of the corresponding mirror projection $\Pi_Z$ plays the role of \cref{lem:P-proj} in the contraction argument.

\section{Additional Numerical Results for the LQ Game}
\label{app:lq-exp}

We briefly summarize additional numerical results for the canonical $64$-dimensional LQ game from \cref{subsec:lq-validation}.
All experiments use the balanced metric $M(w_{\mathrm{bal}})$ with $w_1/w_2=b/a$ and the same coupling grid as in \cref{fig:lq-margins}.

\paragraph{Phase diagrams and flow norms.}
For each $\lambda$ we compute the SGN Euler and RK4 step-size bounds and the true stability thresholds obtained from spectral radii of the one-step matrices.
The resulting phase diagrams on the $(\lambda,h)$-plane are shown in \cref{fig:lq-phase}.
In both cases the SGN step-size curves lie strictly inside the true stability boundaries but remain within a moderate factor of them, confirming that the SGN step-size policies are conservative but nontrivial.
\Cref{fig:lq-flow} shows median norms $\|x(t)\|_{M(w_{\mathrm{bal}})}$ along continuous-time flows for representative couplings, together with reference exponential curves at rates $\alpha_{\mathrm{true}}$ and $\alpha_*$; the empirical decay closely tracks the true metric margin whenever SGN holds.

\begin{figure}[htbp]
  \centering
  \includegraphics[width=1.0\linewidth]{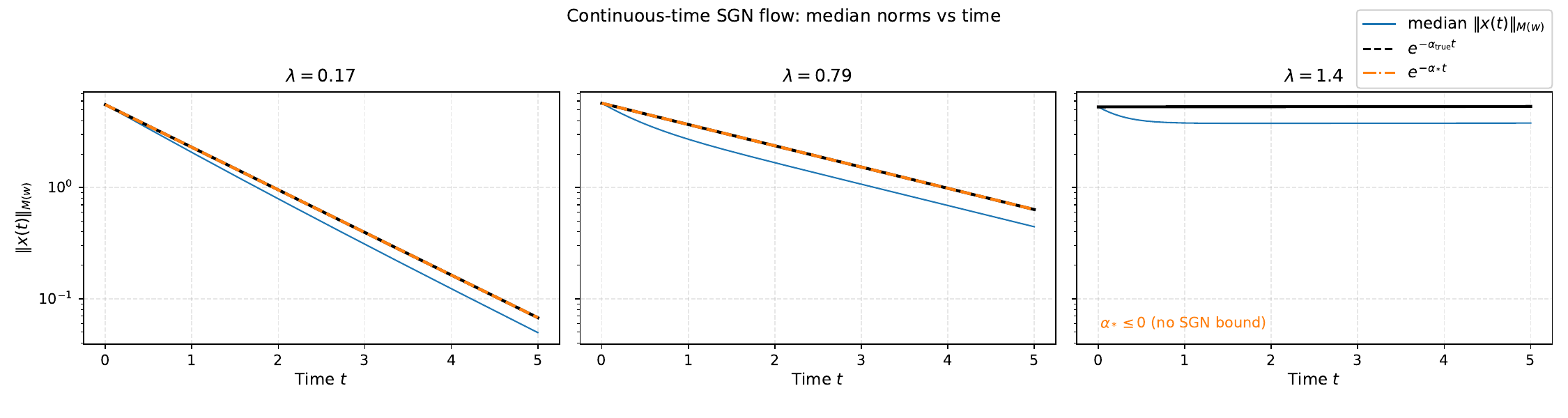}
  \caption{Continuous-time flows for the canonical LQ game in the balanced metric.
  For three representative couplings we plot median norms $\|x(t)\|_{M(w_{\mathrm{bal}})}$ over random initial conditions, together with reference exponentials at rates $\alpha_{\mathrm{true}}$ and $\alpha_*$.
  The empirical decay is consistent with the certified metric margins.}
  \label{fig:lq-flow}
\end{figure}

\paragraph{Structured noise robustness and random LQ ensemble.}
To probe robustness beyond the perfectly structured canonical game we perturb both coupling matrices by spectrally normalized Gaussian noise.
For each noise level $\epsilon\in[0,1]$ we compute the true metric margin and the SGN margin (based on spectral norms of the perturbed couplings) and report the average conservatism ratio $\alpha_{\mathrm{true}}/\alpha_*$ in \cref{fig:lq-noise}.
At $\epsilon=0$ the ratio is close to $1$, while for $\epsilon\to 1$ SGN becomes more conservative but remains a valid lower bound whenever it certifies.

We also ran a small random LQ ensemble, drawing heterogeneous quadratic games with the same $(a,b)$ but random curvature and cross-player blocks, and evaluating Euclidean and SGN margins and RK4 step bounds at several couplings.
With a fixed balanced metric, SGN is intentionally conservative: whenever it certifies a positive margin, the ratio $\alpha_*/\alpha_{\mathrm{true}}$ typically lies between $0.15$ and $1$ over small and moderate couplings, and the logarithmic step-size ratio $\log_{10}(h_{\mathrm{SGN}}/h_{\mathrm{stab}})$ clusters just below zero.
For larger couplings the true metric margin can remain positive while SGN in this fixed metric fails, indicating that metric redesign would be necessary; aggregate statistics and plots are reported in \cref{fig:lq-random}.

\begin{figure}[htbp]
  \centering
  \includegraphics[width=0.6\linewidth]{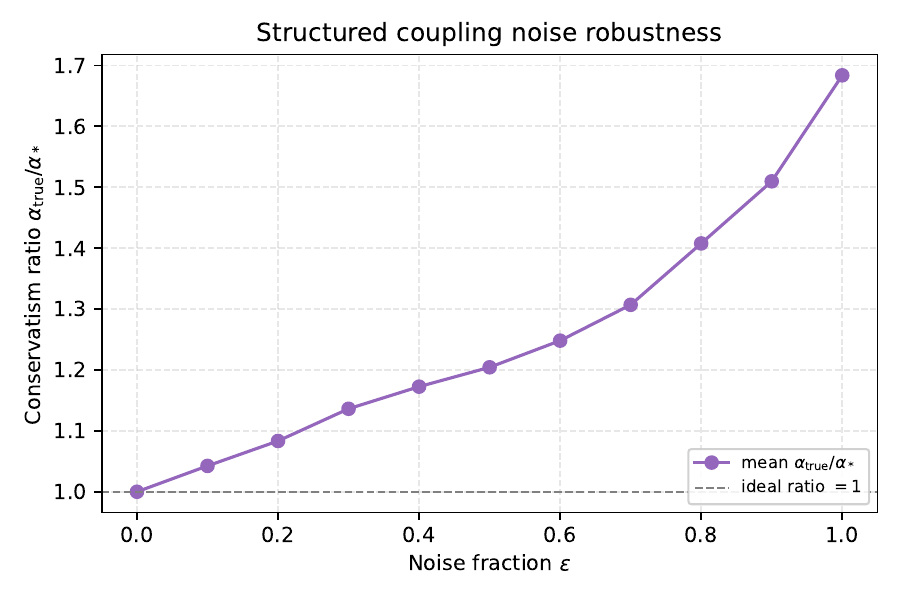}
  \caption{Structured coupling noise robustness for the canonical LQ game.
  The curve shows the average conservatism ratio $\alpha_{\mathrm{true}}/\alpha_*$ as a function of the noise fraction $\epsilon$.
  SGN remains tight in the structured case and becomes gradually more conservative as the couplings become unstructured.}
  \label{fig:lq-noise}
\end{figure}

\begin{figure}[htbp]
  \centering
  \includegraphics[width=0.32\linewidth]{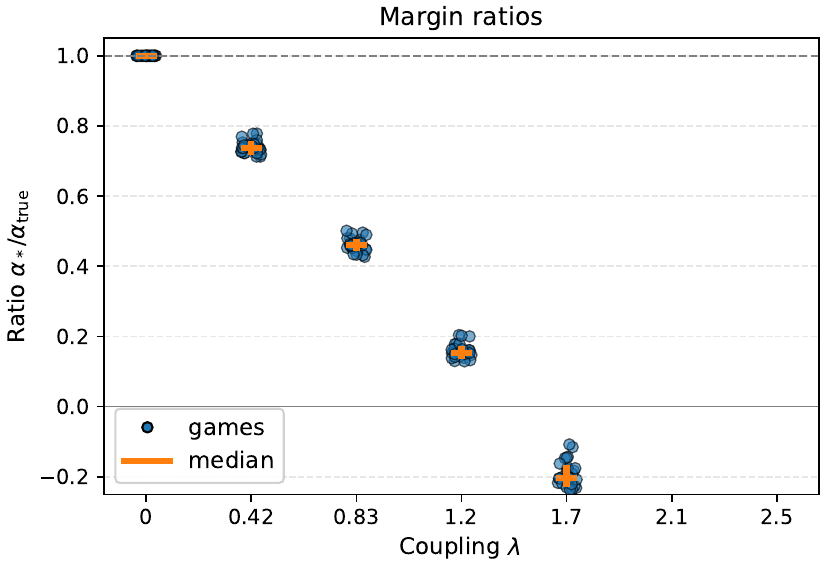}\hfill
  \includegraphics[width=0.32\linewidth]{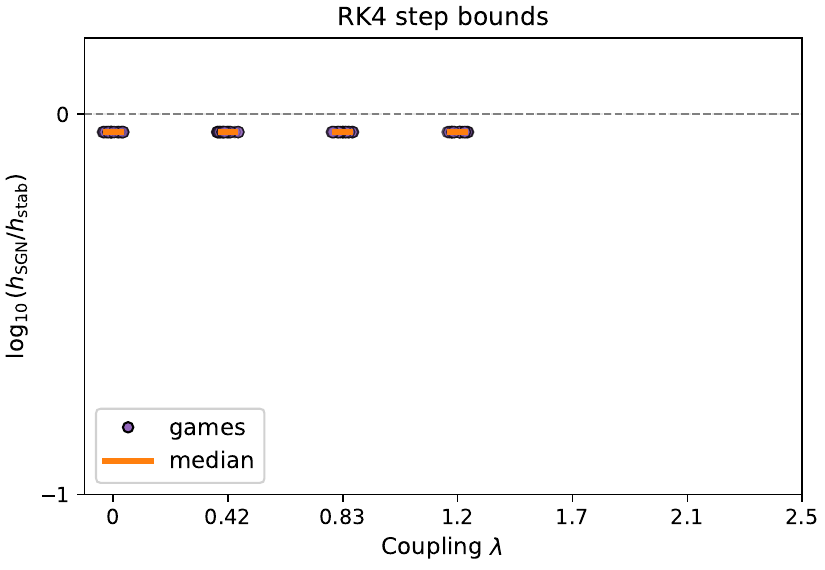}\hfill
  \includegraphics[width=0.32\linewidth]{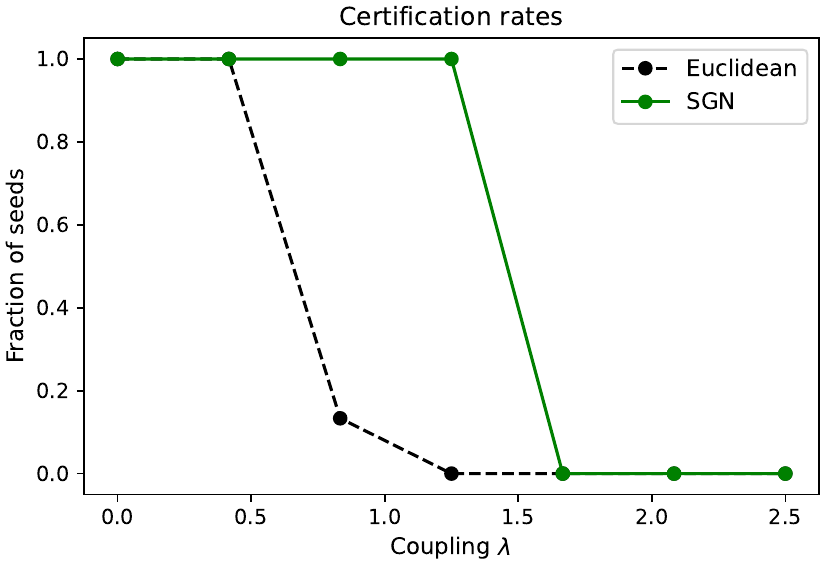}
  \caption{Random LQ ensemble: SGN vs true margins and RK4 step bounds.
  Left: scatter of $\alpha_*/\alpha_{\mathrm{true}}$ per coupling, with medians and interquartile ranges.
  Middle: scatter of $\log_{10}(h_{\mathrm{SGN}}/h_{\mathrm{stab}})$.
  Right: fraction of random instances for which the Euclidean margin is positive vs the fraction for which the SGN margin is positive in the fixed balanced metric.}
  \label{fig:lq-random}
\end{figure}

\end{document}